\documentclass[final,12pt]{alt2022} 

\usepackage{packages}

\title[Efficient and Optimal Fixed-Time Regret with Two
Experts]{Efficient and Optimal Fixed-Time Regret with Two Experts}
\usepackage{times}

\altauthor{%
 \Name{Laura Greenstreet} \Email{laura.greenstreet@gmail.com}\\
 \Name{Nicholas J.\ A.\ Harvey} \Email{nickhar@cs.ubc.ca}\\
 \Name{Victor Sanches Portella} \Email{victorsp@cs.ubc.ca}\\
 \addr University of British Columbia, Department of Computer Science\\
}%

\begin{document}

\maketitle

\begin{abstract}%
    Prediction with expert advice is a foundational problem in online learning. In instances with \(T\) rounds and \(n\) experts, the classical Multiplicative Weights Update method suffers at most \(\sqrt{(T/2)\ln n}\) regret when \(T\) is known beforehand. Moreover, this is asymptotically optimal when both \(T\) and \(n\) grow to infinity. However, when the number of experts \(n\) is small/fixed, algorithms with better regret guarantees exist.  Cover showed in 1967 a dynamic programming algorithm for the two-experts problem restricted to \(\{0,1\}\) costs that suffers at most \(\sqrt{T/2\pi} + O(1)\) regret with \(O(T^2)\) pre-processing time. In this work, we propose an optimal algorithm for prediction with two experts' advice that works even for costs in \([0,1]\) and with \(O(1)\) processing time per turn. Our algorithm builds up on recent work on the experts problem based on techniques and tools from stochastic calculus.
\end{abstract}

\begin{keywords}%
  experts, Cover, online learning, optimal, fixed-time%
\end{keywords}

\section{Introduction}

The foundational problem in online learning of \emph{prediction with expert advice} (or simply \emph{experts' problem}) consists of a sequential game between a player and an adversary. In each turn, the player chooses (possibly randomly) one of \(n\) experts to follow. Concurrently, the adversary chooses for each expert a cost in \([0,1]\). At the end of a turn, the player sees the costs of all experts and suffers the cost of the expert they followed. The performance of the player is usually measured by the \emph{regret}: the difference between their cumulative loss and the cumulative loss of the best expert in hindsight. In this case, we are interested in strategies for the player whose (expected) regret against any adversary is sublinear in the total number of rounds of the game.

A well-known strategy for the player is the \emph{Multiplicative Weights Update} (MWU) method \citep{AroraHK12a}. In the \emph{fixed-time setting} --- that is, when the player knows beforehand the total number of rounds \(T\) --- MWU with a carefully-chosen fixed step-size suffers at most \(\sqrt{(T/2) \ln n}\) regret. Additionally, this regret bound is asymptotically optimal when both \(n\) and \(T\) grow to infinity~\citep{CesaBianchiFHHSW97a}. Yet, if the number of experts \(n\) is fixed/small, better regret guarantees may be possible. 
From a theoretical standpoint, there is a clear motivation for the case where \(n\) is fixed: MWU can suffer regret arbitrarily close to \(\sqrt{(T/2) \ln n}\) for \emph{any} \(n\) as the number of rounds \(T\) grows\footnote{This is known to hold when MWU is used with a fixed or decreasing step-size, which are the usual cases when MWU is applied to the experts' problem. When the step-size of MWU is allowed to be arbitrary, \cite{GravinPS17a} show that MWU can suffer regret arbitrarily close to \((2/3)\sqrt{(T/2) \ln n}\) as the number of rounds \(T\) grows. }~\citep{GravinPS17a}. This means that different ideas are necessary for player strategies to guarantee smaller regret.


\subsection{The Case of Two Experts}
Of course, a natural question is whether regret smaller than \(\sqrt{(T/2) \ln n}\) is even possible as \(T\) grows even if \(n\) is fixed. \cite{Cover67a} showed that for two experts (that is, for \(n = 2\)) a regret of \(\sqrt{T/2\pi} + O(1)\) is the best possible in the worst-case by showing an algorithm for the case with costs in \(\{0,1\}\). In related work, the attainable worst-case regret bounds for three~\citep{AbbasiYadkoriB17a, KobzarKW20a} and four~\citep{BayraktarEZ20a} experts were recently improved, respectively,  to \(\sqrt{8T/(9\pi)} + O(\ln T)\) and \(\sqrt{T \pi/8}\) up to lower-order terms.
Although optimal, Cover's algorithm is based in a dynamic programming approach that takes \(O(T^2)\) time in total. In comparison, MWU takes \(O(1)\) time per round to compute the probabilities to assign to the two experts at each round. Finally, one can adapt Cover's algorithm for costs in \([0,1]\), but the standard approach is to randomly round to costs to either \(0\) or \(1\). In this case, the regret guarantees only hold in expectation.

In a related line of work, \citet{HarveyLPR20a} study the 2-experts problem in the \emph{anytime setting}, that is, in the case where the player/algorithm does not know the total number of rounds \(T\) ahead of time. They showed an \emph{optimal} strategy for the player whose regret on \emph{any} round \(t \in \Naturals\) is at most \((\gamma/2) \sqrt{t}\), where the constant \(\gamma \approx 1.30693\) arises naturally in the study of Brownian motion (see~\citealp{MortersP10a} for an introduction to the field and historical references). Moreover, their algorithm can be computed (up to machine precision) in \(O(1)\) time since it boils down to the evaluation of well-known mathematical functions such as the exponential function and the imaginary error function. In this work, we combine similar ideas based on stochastic calculus together with Cover's algorithm to propose an efficient and optimal fixed-time algorithm for 2-experts.


\paragraph{\textbf{Known result} (\citealp{Cover67a}):} There is a dynamic programming algorithm for the 2-experts problem with costs in \(\{0,1\}\) that suffers at most \(\sqrt{T/2\pi} + 0.5\) regret in games with \(T\) rounds and requires \(O(T^2)\) pre-processing time.

\paragraph{\textbf{Our contribution}:}  An algorithm for the two experts' problem with costs in \([0,1]\) that suffers at most \(\sqrt{T/2\pi} + 1.3\) regret in games with \(T\) rounds. This new algorithm has running time \(O(T)\) and is based on discretizing a continuous-time solution obtained using ideas from stochastic calculus.

More precisely, one of the key steps is deriving a player in the continuous-time setting from~\citet{HarveyLPR20a} that exploits the knowledge of the time-horizon to obtain regret bounds better than in the anytime setting. However, unlike the anytime setting, discretizing this algorithm leads to non-negative discretization error. Another key contribution of our paper is showing that this discretization error is small. Finally, the connections to Cover's classical algorithm sheds new intuition into the classical optimal solution. Interestingly, our results could be formally presented without resorting to stochastic calculus. Yet, it is the stochastic calculus point of view that guides us through the design and analysis of the algorithm.

\paragraph{\textbf{Text organization}:} We first formally define the experts problem and discuss some assumptions and simplifications in Section~\ref{sec:experts}. In Section~\ref{sec:cover_summary} we present a brief summary of Cover's optimal algorithm for two experts. In Section~\ref{sec:continuous_problem} we define an analogous continuous-time problem and describe a solution inspired by Cover's algorithm. Finally, in Section~\ref{sec:discretization} we present and analyze a discretized version of the continuous-time algorithm, showing it enjoys optimal worst-case regret bounds.

 \section{Prediction with Expert Advice}
 \label{sec:experts}

 In this section we shall more precisely define the problem of prediction with expert advice. The problem  is parameterized by a fixed number \(n \in \Naturals\) of experts. A (strategy for the) \textbf{player} is a function \(\Acal\) that, given cost vectors \(\ell_1, \dotsc, \ell_t \in [0,1]^n\) chosen by the adversary in previous rounds, outputs a probability distribution over the \(n\) experts represented by a vector \(x_{t+1} \in \Delta_n \coloneqq \setst{x \in [0,1]^n}{\sum_{i = 1}^n x(i) = 1}\). Similarly, a (strategy for the) \textbf{adversary} is a function \(\Bcal\) that, given previous player's choices \(x_1, \dotsc, x_t \in \Delta_n\) of distributions over the experts, outputs a vector \(c_{t+1} \in [0,1]^n\) of expert costs for round \(t+1\), where \(\ell_{t+1}(i)\) is the cost of expert \(i \in [n] \coloneqq \iinterval{1}{n}\). The performance of a player strategy \(\Acal\) in a game with \(T \in \Naturals\) rounds against an adversary \(\Bcal\) is measured by the \textbf{regret}, defined as
 \begin{equation*}
   \Regret(T, \Acal, \Bcal) \coloneqq \sum_{t = 1}^T \iprodt{\ell_t}{x_t} - \min_{i \in [n]} \sum_{t = 1}^T \ell_t(i),
 \end{equation*}
 where above, and for the remainder of this section\footnote{If no specific strategies \(\Acal\) or \(\Bcal\) are clear from the context, one may take \(\Acal\) and \(\Bcal\) to be arbitrary strategies and we shall omit \(\Acal\) and \(\Bcal\) when they are clear from context.} we have \(x_t \coloneqq \Acal(\ell_1, \dotsc, \ell_{t-1})\) and \(\ell_t \coloneqq \Bcal(x_1, \dotsc, x_{t-1})\) for all \(t \in [T]\). Moreover, whenever the loss vectors \(\ell_1, \dotsc, \ell_T\) are clear from context, we define the \textbf{cumulative loss} of expert \(i \in [n]\) at round \(t \in [T]\) by \(L_{t}(i) \coloneqq \sum_{j = 1}^t \ell_t(i)\). In this text, for each \(T \in \Naturals\) we want to devise a strategy \(\Acal_T\) for the player that suffers regret at most sublinear in \(T\) against \emph{any} adversary in a game with \(T\) rounds. That is, we want a family of strategies \(\{\Acal_T\}_{T \in \Naturals}\) such that
 \begin{equation}
   \label{eq:sublinear_regret}
   \lim_{T \to \infty} \frac{1}{T} \sup_{\Bcal} \Regret(T, \Acal_T, \Bcal) = 0,
 \end{equation}
 where the supremum ranges over all possible adversaries, even those that have full knowledge of (and may even be adversarial to) the player's strategy.


 \subsection{Restricted Adversaries}

 In~\eqref{eq:sublinear_regret}, the supremum ranges over all the possible adversaries for a game with \(T\) rounds. However, we need only consider in the supremum \textbf{oblivious adversaries}~\citep[Section~18.5.4]{KarlinP17a}, that is, adversaries \(\Bcal\) whose choice on each round depends \emph{only} on the round number and not on the choices of the player. For any \(\ell = (\ell_1, \dotsc, \ell_T)^{\transp} \in (\Reals^n)^T\), we denote by \(\Bcal_\ell\) the oblivious adversary that plays \(\ell_t\) on round \(t \in [T]\).

In fact, we may restrict our attention to even smaller sets of adversaries (for details on these reductions, see~\citealp{GravinPS16a} and \citealp[Section~18.5.3]{KarlinP17a}). First, in~\eqref{eq:sublinear_regret} we need only to consider \textbf{binary adversaries}, that is, adversaries which can assign only costs in \(\{0,1\}\) to the experts. Furthermore, to obtain the value of the optimal regret \emph{for two experts} we only need to consider adversaries that pick vector costs in \(\Lcal \coloneqq \curly{(1,0)^\transp, (0,1)^\transp}\), which we call \textbf{restricted binary adversaries}. Intuitively, the adversary can do no better by placing equal costs on both experts at any given round. The optimal algorithm for two experts proposed by \citet{Cover67a} heavily relies on the assumption that the adversary is a restricted binary one and does not extend to general costs in \([0,1]\) without resorting to randomly rounding the costs --- which makes the regret guarantees hold only in expectation.

In this work we design an algorithm the suffers at most \(\sqrt{T/(2\pi)} + O(1)\) regret for arbitrary \([0,1]\) costs. Our initial analysis handles only restricted binary adversaries, but simple concavity arguments extend the upper bound to general adversaries. Throughout this text we fix a time horizon \(T \in \Naturals\).


 \subsection{The Gap Between Experts}

 The case where we have only 2 experts admits a simplification that aids us greatly in the design of upper- and lower-bounds on the optimal regret. Namely, the \textbf{gap} (between experts) at round \(t \in [T]\) is given by \(\abs{L_t(1) - L_t(2)}\), where \(L_t\) is the cumulative loss vector at round \(t\) as defined in Section~\ref{sec:experts}. Furthermore, we denote by \textbf{lagging expert} (on round \(t \in [T]\))  an expert with maximum cumulative loss on round \(t\) among both experts. Similarly, we denote by \textbf{leading expert} (on round \(t \in [T]\)) an expert with minimum cumulative loss on round \(t\).
 The following proposition from~\cite{HarveyLPR20a} shows that, for the restricted binary adversaries described earlier, the regret can be almost fully characterized by the expert gaps and the player's choices of distributions on the experts. In the next proposition (and throughout the remainder of the text), for any predicate \(P\) we define \(\boole{P}\) to be 1 if \(P\) is true and \(0\) otherwise.

 \begin{proposition}[{\citealp[Proposition~2.3]{HarveyLPR20a_arxiv}}]
   \label{prop:gap_regret}
   Fix \(T \in \Naturals\), let \(\Acal\) be a player strategy, and let \(\ell_1, \dotsc, \ell_T \in \{(1,0)^{\transp},(0,1)^{\transp}\}\) be the expert costs chosen by the adversary. For each \(t \in [T]\), set \(x_t \coloneqq \Acal(\ell_1, \dotsc, \ell_{t-1})\),  let \(p_t \in \{x_t(1), x_t(2)\}\) be the probability mass placed on the lagging expert on round \(t\), and let \(g_t\) be the gap between experts on round \(t\). Then,
   \begin{equation*}
     \Regret(T) = \sum_{t = 1}^T \boole{g_{t-1} > 0} p_t \cdot (g_t - g_{t-1}) + \sum_{t = 1}^T
     \boole{g_{t-1} = 0} \iprodt{\ell_t}{x_t},
   \end{equation*}
   where \(g_0 \coloneqq 0\).
   In particular, if for every \(t \in [T]\) with \(g_{t-1} = 0\) we have \(x_t(1) = x_t(2) = 1/2\), then
   \begin{equation*}
     \Regret(T) = \sum_{t = 1}^T p_t \cdot (g_t - g_{t-1}).
   \end{equation*}
 \end{proposition}

\section{An Overview of Cover's Algorithm}
\label{sec:cover_summary}

Although in this section we give only a brief overview of Cover's algorithm, for the sake of completeness we provide a full description and analysis of the algorithm in Appendix~\ref{sec:covers_algorithm}.
The key idea in Cover's algorithm is to compute optimal decisions for all possible scenarios beforehand. This is a feasible approach when we know the total number of rounds and the adversary is a (restricted) binary adversary. More precisely, we will focus our attention to player strategies \(\Acal_p\) parameterized by functions \(p \colon [T] \times \iinterval{0}{T-1} \to [0,1]\)
which place \(p(t,g)\) probability mass on the lagging expert on round \(t\) if the gap between experts is \(g\), and \(1 - p(t,g)\) mass on the leading expert. Then the ``regret-to-be-suffered'' by \(\Acal_p\) at any round \(t\) with a given gap between experts \(g\) is
 \begin{equation}
  \label{eq:Vp_definition_2}
  V_p[t,g] \coloneqq \sup \setst{\Regret(T, \Acal_p, \Bcal_{\ell}) - \Regret(t, \Acal_p, \Bcal_{\ell})}{\ell \in \Lcal^T~\text{s.t.}~\abs{L_{t}(1) - L_{t}(2)} = g}.
\end{equation}
We can compute all entries of \(V_p\) as defined above via a dynamic programming approach, starting with \(V_p[T,g]\) for all \(g \in \iinterval{0}{T-1}\) and then computing these values for earlier rounds. Moreover, there is a simple strategy \(p^*\) that minimizes the worst-case regret \(V_p[0,0]\). Interestingly, the worst-case regret of \(\Acal_{p^*}\) given by \(V^*[0,0]\) is tightly connected with symmetric random walks, where a \textbf{symmetric random walk} (of length \(t\) starting at \(g\)) is a sequence of random variables \((S_i)_{i = 0}^t\) with \(S_i \coloneqq g + X_1 + \dotsm + X_i\) for each \(i \in \iinterval{0}{t}\) and \(\{X_j\}_{j \in [t]}\) are i.i.d.\ uniform random variables on $\{\pm 1\}$. The next theorem summarizes the guarantees on the regret of \(\Acal_{p^*}\), showing that it suffers no more than \(\sqrt{T/2\pi} + O(1)\) regret. Moreover, it is worth noting that no player strategy can do any better asymptotically in \(T\) (for a complete proof of the lower bound, see Appendix~\ref{sec:lower_bound}).

\begin{theorem}[{\citealp{Cover67a}, and \citealp[Section~18.5.3]{KarlinP17a}}]
  For every \(r,g \in \Naturals\), let the random variable \(Z_{r}(g)\) be the number of passages through 0 of a symmetric random walk of length \(r\) starting at position \(g\). Then \(V^*[t,g] = \tfrac{1}{2}\Expect[Z_{T - t}(g)]\) for every \(t,g \in \Naturals\). In particular,
  \begin{equation*}
    V^*[0,0] = \frac{1}{2}\Expect[Z_T(0)] \leq \sqrt{\frac{T}{2 \pi}} + \frac{1}{2}.
  \end{equation*}
\end{theorem}

Finally, although not more efficiently computable than the dynamic programming approach, \(p^*\) has a closed form solution (see~\citealp[Section~18.5.3]{KarlinP17a}) given, for \(t \in [T]\) and~\(g \in \iinterval{0}{T-1}\), by
\begin{equation}
  \label{eq:optimal_p_as_random_walk_prob}
  p^*(t,g) = \frac{1}{2} \Prob(S_{T - t} = g)
  + \Prob(S_{T - t} > g),
\end{equation}
where \(S_{T - t}\) is a symmetric random walk of length \(T - t\). This closed-form solution will serve as inspiration for our continuous-time algorithm.


 \section{A Continuous-Time Problem}
 \label{sec:continuous_problem}

 Cover's player strategy is optimal, but it is defined only for restricted binary adversaries. It is likely that it can be extended to binary adversaries, but it is definitely less clear how to extend such an algorithm for general adversaries picking costs in \([0,1]\). Moreover, even when Cover's algorithm can be used, it is quite inefficient: we either need to compute \(V^*\) which has \(O(T^2)\) entries, or at each round we need to compute the probabilities in~\eqref{eq:optimal_p_as_random_walk_prob}. In the latter case, in the first round we already need \(O(T^2)\) time to exactly compute the probabilities related to a length \(T-1\) random walk.

 To devise a new algorithm for the two experts problem, we first look at an analogous continuous-time problem, first proposed by \citet{HarveyLPR20a} and with a similar setting previously studied by \citet{Freund09a}. The main idea is to translate the random walk connection from the discrete case into a stochastic problem in continuous time, and then exploit the heavy machinery of stochastic calculus to derive a continuous time solution.

 \subsection{Regret as a Discrete Stochastic Integral}

 Let us begin by further connecting Cover's algorithm to random walks. Let \(\Acal_{p}\) be a player strategy induced by some function \(p \colon [T]\times\iinterval{0}{T-1} \to \Reals\). If \(p(t,0) = 1/2\) for all \(t \in \iinterval{0}{T-1}\), then Proposition~\ref{prop:gap_regret} tells us that, for any restricted binary adversary sequence of gaps \(g_1, \dotsc, g_{T} \in \Reals_{\geq 0}\) and for \(g_0 \coloneqq 0\) we have
 \begin{equation}
   \label{eq:discrete_stochastic_integral_regret}
   \Regret(T) = \sum_{t = 1}^T p(t, g_{t-1})(g_t - g_{t-1}).
 \end{equation}
 The right-hand side of the above equation is a discrete analog of the Riemman-Stieltjes integral of \(p\) with respect to \(g\). In fact, if \((g_t)_{t = 0}^{T}\) is a random sequence\footnote{We usually also require some kind of restriction on \((g_t)_{t = 0}^{T}\), such as requiring it to be a martingale or a local martingale.}, the above is also known as a discrete stochastic integral.
 In particular, consider the case where \((g_t)_{t = 0}^{T}\) is a length \(T\) reflected (i.e., absolute value of a) symmetric random walk. Then, any possible sequence of deterministic gaps has a positive probability of being realized by \((g_t)_{t = 0}^{T}\). In other words, any sequence of gaps is in the support of \((g_t)_{t = 0}^T\). Thus, bounding the worst-case regret of \(\Acal_p\) is equivalent to bounding almost surely the value of~\eqref{eq:discrete_stochastic_integral_regret} when~\((g_t)_{t = 0}^T\) is a reflected symmetric random walk. This idea will prove itself powerful in continuous-time even though it is not very insightful for the discrete time problem.

 \subsection{A Continuous-Time Problem}

 A stochastic process that can be seen as the continuous-time analogue of symmetric random walks is \emph{Brownian motion}~\citep{RevuzY99a,MortersP10a}. We fix a Brownian motion \((B_t)_{t \geq 0}\) throughout the remainder of this text.
 Inspired by the observation that the discrete regret boils down to a discrete stochastic integral,~\citet{HarveyLPR20a} define a continuous analogue of regret as a continuous stochastic integral. More specifically, given a function \(p \colon [0,T) \times \Reals \to [0,1]\) such that \(p(t,0) = 1/2\) for all \(t \in \Reals_{\geq 0}\), define the \textbf{continuous regret} at time \(T\) by
 \begin{equation*}
   \ContRegret(T, p) \coloneqq \lim_{\eps \downarrow 0} \int_0^{T - \eps} p(t, \abs{B_t})\diff \abs{B_t},
 \end{equation*}
 where the term in the limit of the right-hand above is the stochastic integral (from \(0\) to \(T - \eps\)) of \(p\) with respect to the process \((\abs{B_t})_{t \geq 0}\). We take the limit as a mere technicality: \(p\) need not be defined at time \(T\) and we want to ensure left-continuity of the continuous regret (the limit is well-defined since a stochastic integral with respect to a reflected Brownian motion is guaranteed to have limits from the left and to be continuous from the right). It is worth noting that usually stochastic integrals are defined with respect to martingales or local martingales, but \((\abs{B_t})_{t \geq 0}\) is neither. Still, \((\abs{B_t})_{t \geq 0}\) happens to be a semi-martingale, which roughly means that it can be written as a sum of two processes: a local-martingale and a process of bounded variation. In this case one can still define stochastic integrals in a way that foundational results such as It\^o's formula still hold and details can be found in~\cite{RevuzY99a}. We do not give the precise definition of a stochastic integral since we shall not deal with its definition directly. Still, one may think intuitively of such integrals as random Riemann-Stieltjes integrals, although the precise definition of stochastic integral is more delicate.


 Let us now look for a continuous function \(p \colon [0, T) \times \Reals \to [0,1]\) with \(p(t,0) = 1/2\) for all \(t \geq 0\) with small continuous regret. Note that without the conditions of continuity or the requirement of \(p(t,0) = 1/2\) for \(t \geq 0\), the problem would be trivial. If we did not require \(p(t,0) = 1/2\) for all \(t \in [0, T)\), then taking \(p(t,g) \coloneqq 0\) everywhere would yield \(0\) continuous regret. Moreover, dropping this requirement would go against the analogous conditions needed in the discrete case, where regret could be written as a ``discrete stochastic integral'' on Proposition~\ref{prop:gap_regret} only when the player chooses \((1/2, 1/2)^{\transp}\) in rounds with \(0\) gap. Finally, requiring continuity of the \(p\) is a way to avoid technicalities and ``unfair'' player strategies.

 When working with Riemann integrals, instead of manipulating the definitions directly we use powerful and general results such as the Fundamental Theorem of Calculus (FTC). Analogously, the following result, known as \emph{It\^o's formula}, is one of the main tools we use when manipulating stochastic integrals and which can be seen as an analogue of the FTC (and shows how stochastic integrals do not always follow the classical rules of calculus).  We denote by \(C^{1,2}\) the class of bivariate functions that are continuously differentiable with respect to their first argument and twice continuously differentiable with respect to their second argument. Moreover, for any function \(f \in C^{1,2}\) we denote by \(\partial_t f\) the partial derivative of \(f\) with respect to its first argument, and we denote by \(\partial_g f\) and \(\partial_{gg} f\), respectively, the first and second derivatives  of \(f\) with respect to its second argument.
 \begin{theorem}[{It\^o's Formula, see~\citealp[Theorem~IV.3.3]{RevuzY99a}}]
   Let \(f \colon [0,T) \times \Reals \to \Reals\) be in \(C^{1,2}\) and let \(T' \in [0,T)\). Then, almost surely,
   \begin{equation}
     \label{eq:ito_formula}
     f(T', \abs{B_{T'}}) - f(0, \abs{B_0})
     = \int_0^{T'} \partial_g f(t, \abs{B_t}) \diff \abs{B_t}
     +  \int_0^{T'} \sqbrac[\big]{\partial_t f(t, \abs{B_t}) + \tfrac{1}{2} \partial_{gg} f(t, \abs{B_t})} \diff t.
   \end{equation}
 \end{theorem}

 Note that the first integral in the equation of the above theorem resembles the definition of the continuous regret. In fact, the above result shows an alternative way to write the continuous regret at time \(T\) of a function \(p \colon  [0, T)\times \Reals \to [0,1]\) such that there is \(R \in C^{1,2}\) with  \(\partial_g R = p\). However, it might be hard to compute (or even to bound) the second integral on~\eqref{eq:ito_formula}. A straightforward way to circumvent this problem is to look for functions such that the second integral in~\eqref{eq:ito_formula} is 0. For that, it suffices to consider functions \(R \in C^{1,2}\) that satisfy the \textbf{backwards heat equation} on \([0, T) \times \Reals\), that is,
 \begin{equation}
   \tag{BHE}
   \label{eq:bhe}
   \bhe R(t, g) \coloneqq \partial_t R(t,g) + \frac{1}{2} \partial_{gg} R(t,g) = 0, \qquad \forall (t,g) \in [0,T) \times \Reals.
 \end{equation}
 We summarize the above discussion and its implications in the following lemma.

 \begin{lemma}
  \label{lemma:summary_ito_experts}
   Let \(R \colon [0, T) \times \Reals \to [0,1]\) be in \(C^{1,2}\) and such that \(\partial_g R(t,g) = p(t,g)\) for all \((t,g) \in [0, T) \times \Reals_{\geq 0}\), such that~\eqref{eq:bhe} holds  and such that \(R(0,0) = 0\). Then \(\lim_{t \uparrow T} R(t, \abs{B_t}) = \ContRegret(T, p)\) almost surely.
 \end{lemma}

 \subsection{A Solution Inspired by Cover's Algorithm}
 In the remainder of this text we will make extensive use of a well-known function related to the Gaussian distribution known as \textbf{complementary error function}, defined by
 \begin{equation*}
   \erfc(z) \coloneqq 1 - \frac{2}{\sqrt{\pi}}\int_{0}^z e^{-x^2} \diff x
   = \frac{2}{\sqrt{\pi}}\int_{z}^{\infty} e^{-x^2} \diff x,
   \qquad \forall z \in \Reals.
 \end{equation*}
 In Section~\ref{sec:bounding_continuous_regret} we will show that the function \(Q \colon (-\infty, T) \times \Reals \to [0,1] \) in \(C^{1,2}\) given by
 \begin{equation*}
   Q(t,g) \coloneqq \frac{1}{2}  \erfc\paren*{\frac{g}{\sqrt{2(T - t)}}}, \qquad \forall (t, g) \in (-\infty, T) \times  \Reals
 \end{equation*}
 satisfies \(\ContRegret(T,Q) = \sqrt{T/(2\pi)}\) almost surely. Before bounding the continuous regret, it is enlightening to see how \(Q\) is related to Cover's algorithm.


 Specifically, let \(p^*\) be as in~\eqref{eq:optimal_p_as_random_walk_prob}. Due to the Central Limit Theorem, \(Q\) can be seen as an approximation of \(p^*\). To see why, let \((S_t)_{t = 0}^{\infty}\) be a symmetric random walk, and define \(X_t \coloneqq S_t - S_{t -1}\) and \(Y_t \coloneqq (X_t + 1)/2\) for each \(t \geq 1\). Note that \(Y_t\) follows a Bernoulli distribution with parameter~\(1/2\) for any \(t \geq 1\). Moreover, let \(Z\) be a Gaussian random variable with mean 0 and variance 1. Then, by setting \(\mu \coloneqq \Expect[2 Y_1] = 1\) and \(\sigma^2 \coloneqq \Expect[(2Y_1 - \mu)^2] = 1\), the Central Limit Theorem guarantees
 \begin{align*}
   \frac{1}{\sqrt{L}}S_L
   &= \frac{1}{\sqrt{L}}\sum_{i = 1}^L X_i
   = \frac{1}{\sqrt{L}} \sum_{i = 1}^L (2Y_i - 1)
  = \frac{\sqrt{L}}{\sigma} \paren*{ \frac{1}{L} \sum_{i = 1}^L 2 Y_i - \mu }
   \stackrel{L \to \infty}{\longrightarrow} Z,
 \end{align*}
 where the limit holds in distribution. Thus, we roughly have that \(S_L\) and \(\sqrt{L} Z\) have similar distributions. Then,
 \begin{align*}
   p^*(t,g)
   &= \frac{1}{2}\Prob(S_{T - t} = g) + \Prob(S_{T - t} > g)
   \approx \frac{1}{2}\Prob((\sqrt{T - t}) Z = g) + \Prob((\sqrt{T - t}) Z > g)
   \\
   &
   = \Prob\paren*{Z > \frac{g}{\sqrt{T - t}}}
   = \frac{1}{2} \erfc\paren*{\frac{g}{\sqrt{2(T - t)}}}
   = Q(t,g).
 \end{align*}
 One may already presume that using \(Q\) in place of \(p^*\) in the discrete experts' problem should yield a regret bound close enough to the guarantees on the regret of Cover's algorithm. Indeed, using Berry-Esseen's Theorem~\citep[Section~3.4.4]{Durret19a} to more precisely bound the difference between \(p^*\) and \(Q\) yields a \(O(\sqrt{T})\) regret bound with suboptimal constants against binary adversaries. However, it is not clear if the approximation error would yield the optimal constant in the regret bound. Additionally, these guarantees do not naturally extend to arbitrary experts' costs in \([0,1]\). In Section~\ref{sec:discretization} we will show how to use an algorithm closely related to \(Q\) that enjoys a clean bound on the discrete-time regret.

 \paragraph{Deriving \(Q\) directly from a PDE.} We have derived \(Q\) by a heuristic argument to approximate \(p^*\). Yet, one can derive the same solution without ever making use of \(p^*\) by approaching the problem directly from the stochastic calculus point of view. Namely, consider player strategies that satisfy the BHE, are non-negative, and that place \(1/2\) mass on each expert when the gap is \(0\). With only these conditions we would end up with anytime solutions similar to the ones considered by~\citet{HarveyLPR20a}. In the fixed-time case we can ``invert time'' by a change of variables \(t \gets T - t\). Then the BHE becomes the traditional heat equation, which \(Q\) satisfies together with the boundary conditions.

 \subsection{Bounding the Continuous Regret}
 \label{sec:bounding_continuous_regret}

 Interestingly, not only is \(Q\) in \(C^{1,2}\), but it also satisfies the backwards heat equation, even though we have never explicitly required such a condition to hold. Since the proof of this fact boils down to technical but otherwise straightforward computations, we defer it to Appendix~\ref{app:missing_proofs_continuous_regret}.

 \begin{restatable}{lemma}{BHEForQ}
   For all \(t \in [0, T)\) and \(g \in \Reals_{\geq 0}\) we have \(\bhe Q(t,g) = 0\).
 \end{restatable}

 However, recall that to use Lemma~\ref{lemma:summary_ito_experts} we need a function \(R \in C^{1,2}\) with \(\partial_g R = Q\) that satisfies the backwards heat equation, not necessarily \(Q\) itself needs to satisfy the backwards heat equation. Luckily enough, the following lemma shows how to obtain such a function \(R\) based on \(Q\).

 \begin{proposition}[{\citealp[Lemma~5.6]{HarveyLPR20a_arxiv}}]
   Let \(h \colon [0, T) \times \Reals \to \Reals\) be in \(C^{1,2}\) and define
   \begin{equation*}
     f(t,g)
     \coloneqq \int_0^g h(t,y) \diff y
     - \frac{1}{2} \int_0^t \partial_g h(s, 0) \diff s,
     \qquad \forall (t,g) \in [0, T) \times \Reals.
   \end{equation*}
   Then,
   \begin{enumerate}[(i)]
     \item \(f \in C^{1,2}\),
     \item If \(h\) satisfies~\eqref{eq:bhe}, then so does \(f\),
     \item \(h = \partial_g f\).
   \end{enumerate}
 \end{proposition}

 In light of the above proposition, for all \((t,g) \in (-\infty, T) \times \Reals\) define
 \begin{equation*}
   R(t,g) \coloneqq \int_0^g Q(t,x) \diff x - \frac{1}{2} \int_0^t \partial_g Q(s,0) \diff s.
 \end{equation*}
 In the case above, we can evaluate these integrals and obtain a formula for \(R\) that is easier to analyze. Although we defer a complete proof of the next equation to Appendix~\ref{app:missing_proofs_continuous_regret}, using that \(\int_0^y \erfc(x) \diff x = y \erfc(y) - \frac{1}{\sqrt{\pi}} e^{-y^2} + \frac{1}{\sqrt{\pi}}\) \citep[Section~7.7(i)]{NIST10a} and that \(\frac{\diff}{\diff x} \erfc(x) = - \frac{2}{\sqrt{\pi}} e^{-x^2} \) we can show for every \(t \in (-\infty, T)\) and $g \in \Reals$ that 
 \begin{equation}
  \label{eq:rtg_formula}
  R(t,g) = \frac{g}{2} \erfc\paren[\Bigg]{\frac{g}{\sqrt{2(T - t)}}} - \sqrt{\frac{T -t}{2 \pi}} \exp\paren[\Bigg]{-\frac{g^2}{2(T - t)}} + \sqrt{\frac{T}{2 \pi}}.
 \end{equation}
 Since \(R\) satisfies~\eqref{eq:bhe}, Lemma~\ref{lemma:summary_ito_experts} shows the continuous regret of \(Q\) is given exactly by \(R\). The following lemma shows a bound on \(R\) and, thus, a bound on the continuous regret of \(Q\).

 \begin{lemma}
   \label{lemma:properties_of_R}
   We have \(R(0,0) = 0\) and
   \begin{equation*}
      R(t,g) \leq \sqrt{\frac{T}{2\pi}}, \qquad \forall (t,g) \in [0,T) \times \Reals.
   \end{equation*}
 \end{lemma}
 \begin{proof}
   The facts that \(R(0,0) = 0\) and that \(R(t,g) \leq 0\) for \(g \leq 0 \) and \(t \in [0,T)\) are easily verifiable. For the bound on \(R\) for \(g > 0\), note first that for any \(z > 0\) we have
   \begin{equation*}
     \erfc(z)
     = \frac{2}{\sqrt{\pi}} \int_{z}^{\infty} e^{x^2} \diff x
     = \frac{2}{\sqrt{\pi}} \int_{z}^{\infty}
     \frac{2x}{2x} e^{x^2} \diff x
     \leq
     \frac{1}{z\sqrt{\pi}} \int_{z}^{\infty}
     2x e^{x^2} \diff x
     = \frac{e^{-z^2}}{z\sqrt{\pi}}.
   \end{equation*}
   Therefore, for all \((t,g) \in [0, T) \times \Reals_{> 0}\) we have
   \begin{equation*}
     \frac{g}{2} \erfc\paren[\Bigg]{\frac{g}{\sqrt{2(T - t)}}}
     \leq \frac{g}{2} \cdot \frac{\sqrt{2(T - t)}\exp\paren*{\frac{-g^2}{2(T-t)}}}{g \sqrt{\pi}}
     = \sqrt{\frac{T-t}{2 \pi}} \exp\paren*{\frac{-g^2}{2(T-t)}}.
   \end{equation*}
   Applying the above to~\eqref{eq:rtg_formula} yields the desired bound.
 \end{proof}

 Combining these results we get the desired bound on the continuous regret of \(Q\), which we summarize in the following theorem.

 \begin{theorem}
   We have \(\ContRegret(T,Q) \leq \sqrt{T/(2\pi)}\) almost surely.
 \end{theorem}

 \label{sec:continuous_algorithm}

 \section{From Continuous to Discrete Time}
 \label{sec:discretization}


 In the continuous time algorithm we have that \(R(t,g)\) is the continuous regret at time \(t\) with gap \(g\) of the strategy that places probability mass on the lagging expert\footnote{We have never formally defined lagging and leading experts in continuous time, and we do not intend to do so. Here we are extrapolating the view given by Proposition~\ref{prop:gap_regret} of regret as a stochastic integral of the probability put on the lagging expert with respect to the gaps for the sake of intuition.} according to~\(Q(t,g) = \partial_g R(t,g)\). At the same time, for Cover's algorithm we have \(V^*[t,g]\) as an upper-bound on the regret when the mass on the lagging expert is given by \(p^*(t,g)\). Furthermore, similar to the relation between \(Q\) and~\(R\), we can write \(p^*\) as a function of \(V^*\) (details can be found on Appendix~\ref{sec:covers_algorithm}): at round \(t\) with gap \(g\) at round \(t-1\), the probability mass placed on the lagging expert in Cover's algorithm is\footnote{For \(g = 0\) this does not follow directly, but our goal at the moment is only to build intuition. }
 \begin{equation*}
   p^*(t,g) = \frac{V^*[t, g-1] - V^*[t, g+1]}{2} \approx \partial_g V^*[t,g].
 \end{equation*}
 That is, \(p^*\) is a sort of discrete derivative of \(V^*\) with respect to its second argument. From this analogy, one might expect that a discrete derivative of \(R\) with respect to its second argument yields a good strategy for the player in the original experts' problem. As we shall see, this is exactly the case. Additionally, computing the discrete derivative of \(R\) amounts to a couple of evaluations of the complementary error function, which we can assume to be computable (up to machine precision) in constant time.

 In this section we shall describe the discretized algorithm and give an upper-bound on its regret against restricted binary adversaries, that is, adversaries that choose costs in \(\curly{(0,1)^{\transp}, (1,0)^{\transp}}\). Luckily, unlike Cover's algorithm, the strategy we shall see in this section smoothly extends to general costs in \([0,1]\) while preserving its performance guarantees. Since the details of this extension amounts to concavity arguments, we defer the details of this extension to Appendix~\ref{app:general_costs}.

 \subsection{Discrete It\^o's Formula}

 In Section~\ref{sec:continuous_problem}, the main tool to relate the continuous regret to the function \(R\) was It\^o's formula. Similarly, one of the main tools for the analysis of the discretized continuous-time algorithm will be a discrete version of It\^o's formula. In order to state such a formula and to describe the algorithm, some standard notation to denote discrete derivative will be useful. Namely, for any function \(f \colon \Reals^2 \to \Reals\) and any \(t,g \in
 \Reals\), define
 \begin{align*}
   f_g(t,g) &\coloneqq
   \frac{f(t, g + 1) - f(t, g -1)}{2},\\
   f_t(t,g) &\coloneqq f(t,g) - f(t-1,g),\\
   f_{gg}(t,g)
   &\coloneqq f(t, g+1) + f(t, g-1) - 2f(t,g).
 \end{align*}

 We are now in place to state a discrete analogue of It\^o's formula. One important assumption of the next theorem is that \(g_0, \dotsc, g_T \in \Reals\) are such that successive values have absolute difference equal to $1$. In the case where \(g_0, \dotsc, g_T\) are gaps in a 2-experts problem, this means that the adversary needs to be a restricted binary adversary. The version of the next theorem as stated --- including the dependence on \(t\) --- can be found in \citet[Lemma~3.7]{HarveyLPR20a}. Yet, this theorem is a slight generalization of earlier results such as the ones due to \citet[Section~2]{Fujita08a} and \citet[Theorem~2]{Kudzhma82a} 

 \begin{theorem}[{Discrete It\^o's Formula}]
   \label{thm:discrete_ito}
   Let \(g_0, g_1, \dotsc, g_T \in \Reals\) be such that $\abs{g_t -
   g_{t-1}} = 1$ for every \(t \in [T]\) and let \(f \colon \Reals^2 \to
   \Reals\). Then,
   \begin{equation*}
     f(T, g_T) - f(0, g_0)
     = \sum_{t = 1}^T f_g(t, g_{t-1}) (g_t - g_{t-1})
     + \sum_{t = 1}^T \paren[\big]{\tfrac{1}{2} f_{gg}(t, g_{t-1}) +
     f_{t}(t, g_{t-1})}.
   \end{equation*}
 \end{theorem}

 The first summation in the right-hand side of discrete It\^o's formula can be seen as a discrete stochastic integral when \((g_t)_{t = 0}^T\) is a random sequence. Remarkably, this term is extremely similar to the regret formula from Proposition~\ref{prop:gap_regret}. Thus, if we were to use discrete It\^o's formula to bound the regret, it would be desirable for the second term to (approximately) satisfy an analogue of~\eqref{eq:bhe}. In fact, the potential \(V^*\) from Cover's algorithm satisfies the discrete BHE (with some care needed when the gap is zero, see Appendix~\ref{app:dbhe}). Furthermore, the connection between BHE seems to extend to other problems in online learning: in recent work, \citet{ZhangCP22a} showed how coin-betting with potentials that satisfy the BHE yield optimal algorithms for unconstrained online learning.

 Since \(R\) satisfies~\eqref{eq:bhe}, one might hope that \(R\) would also satisfy such a discrete backwards-heat inequality, yielding an upper-bound on the regret of the strategy given by~\(R_g\). In the work of \citet{HarveyLPR20a} in the anytime setting, it was the case that the terms in the second sum were non-negative, which in a sense means that the discretized algorithm suffers \emph{negative} discretization error. In the fixed-time setting we are not as lucky.

 \subsection{Discretizing the Algorithm}

 Based on the discussion at the beginning of this section, a natural way to discretize the algorithm from Section~\ref{sec:continuous_problem} is to define the function \(q \colon [T] \times \iinterval{0}{T-1} \to \Reals\) by
 \begin{equation*}
   q(t,g) \coloneqq
   \begin{cases}
     R_g(t,g)
     &\text{if}~t < T,\\
     \boole{g = 0} \frac{1}{2}   &\text{if}~t = T,
   \end{cases}
    \qquad \forall t \in [T], \forall g \in \iinterval{0}{T-1},
 \end{equation*}
 where we need to treat the case at the very last step differently since \(R\) is not defined on \(\{T\} \times \Reals\). It is not clear from its definition, but we indeed have \(q(t,0) = 1/2\) for all \(t \in [T]\). We defer the (relatively technical) proof of the next result to Appendix~\ref{sec:missing_proofs_discretization}.

 \begin{restatable}{lemma}{qIsOneHalf}
   \label{lemma:q_is_one_half}
   We have \(q(t,0) = 1/2\) for all \(t \in [T]\).
 \end{restatable}

 Our goal now is to combine Proposition~\ref{prop:gap_regret} and the discrete It\^o's formula to bound the regret of~\(\Acal_q\). Since \(R\) satisfies the~\eqref{eq:bhe}, one might hope that
 \(R\) is close to satisfying the discrete version of this equation. To formalize this idea, for all \(t \in (-\infty,T)\) and \(g \in \Reals\) define
 \begin{equation*}
   r_{gg}(t,g) \coloneqq \partial_{gg}R(t,g) - R_{gg}(t,g)
   \qquad \text{and} \qquad
   r_{t}(t,g) \coloneqq \partial_{t}R(t,g) - R_{t}(t,g).
 \end{equation*}
 The above terms measure how well the first derivative with respect to the first variable and the second derivative with respect to the second variable are each approximated by their discrete analogues. That is, these are basically the discretization errors on the derivatives of \(R\).   Then, combining the fact that \(R\) satisfies~\eqref{eq:bhe} together with Proposition~\ref{prop:gap_regret} yields the following theorem.

 \begin{theorem}
   \label{thm:regret_with_discretization_error}
   Consider a game of \(\Acal_q\) with a restricted binary adversary with gap sequence given by \(g_0, g_1, g_2, \dotsc, g_T \in \iinterval{0}{T}\) such that \(g_0 = 0\) and \(\abs{g_t - g_{t-1}} = 1\) for all \(t \in [T]\). Then,
   \begin{equation}
     \label{eq:regret_with_ito}
     \Regret(T)
     \leq \sqrt{\frac{T}{2 \pi}} + \frac{1}{2}
     + \frac{1}{2}\sum_{t = 1}^{T-1} r_{gg}(t, g_{t-1})
     + \sum_{t = 1}^{T-1} r_t(t, g_{t-1}).
   \end{equation}
 \end{theorem}
 \begin{proof}
   Lemma~\ref{lemma:q_is_one_half} and Proposition~\ref{prop:gap_regret} yield
   \begin{equation}
     \label{eq:proof_regret_with_ito_1}
     \Regret(T)
     = \sum_{t = 1}^T q(t, g_{t-1})(g_t - g_{t-1})
     \leq \sum_{t = 1}^{T-1} q(t, g_{t-1})(g_t - g_{t-1}) + \frac{1}{2},
   \end{equation}
   where in the last inequality we used that \(q(T,g_{T-1}) \leq 1/2\).
   Furthermore, by the discrete It\^o's formula (Theorem~\ref{thm:discrete_ito}), we have
 \begin{align*}
   R(T-1, g_{T-1}) - R(0, g_0)
   &= \sum_{t = 1}^{T-1} R_g(t, g_{t-1}) (g_t - g_{t-1})
   + \sum_{t = 1}^{T-1} \paren[\big]{\tfrac{1}{2} R_{gg}(t, g_{t-1}) +
   R_{t}(t, g_{t-1})} \\
   &\stackrel{\eqref{eq:bhe}}{=} \sum_{t = 1}^{T-1} q(t, g_{t-1}) (g_t - g_{t-1})
   - \sum_{t = 1}^{T-1} \paren[\big]{\tfrac{1}{2} r_{gg}(t, g_{t-1}) +
   r_{t}(t, g_{t-1})}
   \\
   &
   \stackrel{\eqref{eq:proof_regret_with_ito_1}}{\geq}
    \Regret(T) - \frac{1}{2} - \sum_{t = 1}^{T-1} \paren[\big]{\tfrac{1}{2} r_{gg}(t, g_{t-1}) +
   r_{t}(t, g_{t-1})}.
 \end{align*}
 Rearranging and using the facts given by Lemma~\ref{lemma:properties_of_R} that \(R(0,0) = 0\) and that \(R(T-1, g_{T-1}) \leq \sqrt{T/(2\pi)}\) yield the desired bound on the regret.
 \end{proof}

 \subsection{Bounding the Discretization Error}

 In light of Theorem~\ref{thm:regret_with_discretization_error}, it suffices to bound the accumulated discretization error of the derivatives to obtain potentially good bounds on the regret of \(\Acal_q\). The next two lemmas show that both \(r_t(t,g)\) and \(r_{gg}(t,g)\) are in \(O((T - t)^{-3/2})\). Since
 \begin{equation}
  \label{eq:integral_bound}
   \sum_{t = 1}^{T-1} (T - t)^{-3/2}  \leq \int_{t = 0}^{T-1} (T - t)^{-3/2} \diff t = 2\paren*{1 - \frac{1}{\sqrt{T}}} \leq 2,
 \end{equation}
 this will show that \(\Acal_q\) suffers at most \(\sqrt{T/2\pi} + O(1)\) regret.\footnote{{This together  with Prop.~\ref{prop:gap_regret} also shows that the difference between in the regret of \(\Acal_q\) and \(\Acal_Q\) is in \(O(1)\).}} Since the proof of these bounds are relatively technical but otherwise not considerably insightful, we defer them to Appendix~\ref{sec:missing_proofs_discretization}.

 \begin{restatable}{lemma}{BoundOnrt}
   \label{lemma:bound_rt}
   For any \(t \in (-\infty, T)\) and \(g \in \Reals\) we have
   \begin{equation*}
     r_t(t,g) \leq \frac{\sqrt{2}}{8 \sqrt{\pi}} \cdot\frac{1}{(T - t)^{3/2}}
     \qquad \text{and} \qquad
     r_{gg}(t,g) \leq \frac{2 \sqrt{2}}{3\sqrt{\pi}} \cdot\frac{1}{(T - t)^{3/2}}.
   \end{equation*}
 \end{restatable}

 Combining the above lemmas together with Theorem~\ref{thm:regret_with_discretization_error} yields the following regret bound.

 \begin{theorem}
   \label{thm:final_theorem}
   Define \(q(t,g) \coloneqq R_g(t,g)\) for all \((t,g) \in \iinterval{0}{T-1}^2\), consider a game of \(\Acal_q\) against a restricted binary adversary. Then,
   \begin{equation*}
    \Regret(T)
   \leq \sqrt{\frac{T}{2 \pi}} + 1.24.
   \end{equation*}
 \end{theorem}
 \begin{proof}
  Let \(g_1, g_2, \dotsc, g_T \in \iinterval{0}{T}\) be the gap sequence and set \(g_0 \coloneqq 0\).
   We have
 \begin{align*}
   \Regret(T)
   &\leq \sqrt{\frac{T}{2 \pi}} + \frac{1}{2}
   + \frac{1}{2}\sum_{t = 1}^{T-1} r_{gg}(t, g_{t-1})
   + \sum_{t = 1}^{T-1} r_t(t, g_{t-1})
   &\text{(by Theorem~\ref{thm:regret_with_discretization_error}),}
   \\
   &\leq
   \sqrt{\frac{T}{2 \pi}} + \frac{1}{2}
   + \paren*{ \frac{\sqrt{2}}{8 \sqrt{\pi}}
   + \frac{\sqrt{2}}{3\sqrt{\pi}}
   }\sum_{t = 1}^{T-1}\frac{1}{(T-t)^{3/2}}
   &\text{(by Lemma~\ref{lemma:bound_rt}),}
   \\
   &\leq \sqrt{\frac{T}{2 \pi}} + \frac{1}{2}
   + \paren*{ \frac{\sqrt{2}}{4 \sqrt{\pi}} + \frac{2\sqrt{2}}{3\sqrt{\pi}} }
   \leq \sqrt{\frac{T}{2 \pi}} + 1.24,
   & \text{(by~\eqref{eq:integral_bound}).}&
 \end{align*}
 \end{proof}

\section{On Optimal Regret for More than Two Experts}


In this paper we presented an efficient and optimal algorithm for two experts in the fixed time-setting. A natural question is whether similar techniques can be used to find the minimax regret when we have more than two experts. Encouragingly, techniques from stochastic calculus were also used to find the optimal regret for 4 experts~\citep{BayraktarEZ20a}. Yet, it is not clear how to use similar techniques for cases with arbitrary number of experts. The approach used in this paper and by~\citet{HarveyLPR20a} heavily relies on the gap parameterization for the problem. Although there is an analogous parameterization of the $n$ experts’ problem into $n - 1$ gaps that yields a claim similar to Proposition~\ref{prop:gap_regret}, it is not clear what would be an analogous continuous-time problem to guide us in the algorithm design process since the gap process \emph{are not independent}---even with independent costs on the experts. Moreover, many of the approaches in related work~\citep{AbbasiYadkoriB17a,HarveyLPR20a,BayraktarEZ20a} focus on specific adversaries such as the \emph{comb adversary}. However, the latter does not seem to be a worst-case adversary for cases such as for five experts~\citep{Chase19a}. We are not aware of adversaries that could yield worst-case regret for arbitrary \emph{fixed} number of experts, although asymptotically in \(n\) and \(T\) it is well-known that assigning \(\{0,1\}\) costs at random is minimax optimal~\citep{Cesa-BianchiL06a}


%
\acks{We would like to thank the anonymous ICML 2022 reviewers for their insightful comments. In particular, reviewer 1 suggested the use of Berren-Esseen-like results to derive \(O(\sqrt{T})\) regret, noted the \(O(1)\) regret difference between \(\Acal_q\) and \(\Acal_Q\), and found a calculation mistake.

\noindent
N.~Harvey was supported by an NSERC Discovery Grant.}

\bibliography{ref.bib}

\appendix

\section{Cover's Algorithm for Two Experts}

\label{sec:covers_algorithm}

In this section, we shall review the optimal algorithm for the 2-experts problem originally proposed by~\citet{Cover67a} and the matching lower-bound.

\subsection{A Dynamic Programming View}

In the fixed-time setting we know the total number of rounds before the start of the game. Thus,  we may compute ahead of time all the possible states the game can be on each round and decide the probabilities on the experts the player should choose in each case to minimize the worst-possible regret. More specifically, we start with a function \(p \colon [T] \times \iinterval{0}{T-1} \to [0,1]\) that represents our player strategy: for any \(t \in [T]\), if at the end of a round \(t - 1\) the experts' gap is \(g \in \iinterval{0}{T-1}\), on round \(t\) the player places \(p(t,g)\) probability mass on the lagging\footnote{When the gap is 0, which means that both experts have the same cumulative loss, we break ties arbitrarily. For the optimal algorithm we shall ultimately derive this will not matter since we will have \(p(t,0) = 1/2\) for all \(t \in [T]\).} expert and \(1 - p(t,g)\) probability on the leading expert, and we denote by \(\Acal_p\) such a player strategy defined by \(p\). Now, for all \(t \in \iinterval{0}{T}\) and \(g \in \iinterval{0}{T}\), denote by \(V_p[t,g]\) the maximum regret-to-be-suffered by the player strategy defined by \(p\) on rounds \(t+1, \dotsc, T\) given that at the end of round \(t\) the gap between experts is \(g\). Slightly more formally, we have
\begin{equation}
  \label{eq:Vp_definition}
  V_p[t,g] \coloneqq \sup \setst{\Regret(T, \Acal_p, \Bcal_{\ell}) - \Regret(t, \Acal_p, \Bcal_{\ell})}{\ell \in \Lcal^T~\text{such that}~\abs{L_{t}(1) - L_{t}(2)} = g},
\end{equation}
where \(\Lcal \coloneqq \curly{(1,0)^{\transp}, (0,1)^{\transp}}\). Above we take the supremum instead of the maximum only to account for cases where the set we are considering is empty (and, thus, the supremum evaluates to \(-\infty\)), such as when \(t\) and \(g\) have distinct parities or when \(g > t\). Note that by definition of \(V_p\) we have
\begin{equation*}
  V_p[0,0] = \max_{\ell \in \Lcal^T} \Regret(T, \Acal_p, \Bcal_\ell).
\end{equation*}
Thus, if we compute \(V_p[0,0]\), then we have a bound on the worst-case regret of \(\Acal_{p}\) against  restricted binary adversaries. The following proposition shows how we can compute this value in a dynamic programming style.
\begin{theorem}
  \label{thm:dp_formula}
  For any \(p \colon [T] \times \iinterval{0}{T-1} \to [0,1]\), and for all \(t,g \in \{0, \dotsc, T\}\) such that \(V_p[t,g] \neq -\infty\) we have
  \begin{alignat}{4}
    \label{eq:dp_1}
    V_p[t, g] & = 0 &\qquad \text{if}~t = T,
    \\
    \label{eq:dp_2}
    V_p[t,g] &= \max
    \begin{cases}
     V_p[t+1, g+1] + p(t+1,g)
     \\ V_p[t+1, g -1 ] - p(t+1,g)
    \end{cases}
    & \qquad  \text{if}~t < T~\text{and}~g > 0,
    \\
    \label{eq:dp_3}
    V_p[t,g] &= V_p[t+1, 1] + \max\{p(t+1,0), 1 - p(t+1, 0)\}&\qquad  \text{if}~t < T~\text{and}~g = 0,.
  \end{alignat}
\end{theorem}
\begin{proof}
  First, note that~\eqref{eq:dp_1} clearly holds by the definition of \(V_p\). To show that equations~\eqref{eq:dp_2} and~\eqref{eq:dp_3} hold, let \(t,g \in \{0, \dotsc, T\}\) be such that \(V_p[t,g] \neq -\infty\). Let \(\ell \in \Lcal^{T}\) be a sequence of cost vectors such that we have~\( g_t \coloneqq \abs{L_t(1) - L_t(2)} = g\). First, suppose \(t < T\) and \(g > 0\). Then there are two cases for \(\ell\): either the gap \(g_{t+1} \coloneqq \abs{L_{t+1}(1) - L_{t+1}(2)}\) goes up and \(g_{t+1} = g_t + 1\), or it goes down and \(g_{t+1} = g_t - 1\). This together with Proposition~\ref{prop:gap_regret} and the fact that \(g_t = g > 0\) implies
  \begin{align*}
    \Regret(T) - \Regret(t)
    &= \Regret(T) - \Regret(t + 1) + p(t+1,g_t) (g_{t + 1} - g_{t})
    \\
    &= \begin{cases}
      \Regret(T) - \Regret(t + 1) + p(t+1,g), &\text{if}~g_{t+1}= g + 1\\
      \Regret(T) - \Regret(t + 1) - p(t+1,g), &\text{if}~g_{t+1}= g - 1.
    \end{cases}
  \end{align*}
  By taking the maximum over all possible cost vectors with gap \(g\) at round \(t\) we obtain~\eqref{eq:dp_2}. Now suppose \(t < T\) and \(g = 0\). In this case, suppose without loss of generality that \(1\) is the expert to whom \(\Acal_p\) assigns mass \(p(t,0)\) (recall that the strategy \(\Acal_p\) breaks ties arbitrarily when the gap is \(0\)). Proposition~\ref{prop:gap_regret} together with the fact that \(\ell_t \in \Lcal = \{(1,0)^{\transp}, (0,1)^{\transp}\}\)
  \begin{align*}
    \Regret(T) - \Regret(t)
    &= \Regret(T) - \Regret(t + 1) + \iprodt{\ell_t}{x_t}
    \\
    &= \begin{cases}
      \Regret(T) - \Regret(t + 1) + p(t+1,g) &\text{if}~\ell_t(1) = 1,\\
      \Regret(T) - \Regret(t + 1) + 1 - p(t+1,g) &\text{if}~\ell_t(1) = 0.
    \end{cases}
  \end{align*}
  Since the gap on round \(t+1\) is certainly \(1\) in this case, taking the maximum over all the adversaries with gap \(0\) on round \(t\) yields~\eqref{eq:dp_3}.
\end{proof}

For the sake of convenience, we redefine \(V_p[t,g]\) for all \(t,g\) such that \(V_{p}[t,g] = -\infty\) to, instead, be the value given by the equations from the above theorem.\footnote{There will be places where this definition requires access to undefined or ``out-of-bounds'' entries (such as for entries with gap \(T\) and time \(t < T\)). In such cases, we set such undefined/out-of-bounds values to \(0\).} This does not affect any of our results and makes it less cumbersome to design and analyze the algorithm.

\subsection{Picking Optimal Probabilities}

We are interested in a function \(p^* \colon [T] \times \{0, \dotsc, T-1\} \to [0,1]\), if any, that minimizes \(V_p[0,0]\). To see that there is indeed such a function, note that we can formulate the problem of minimizing \(V_p[0,0]\) as a linear program using Theorem~\ref{thm:dp_formula} to design the constraints. Such a linear program is certainly bounded (the regret is always between \(0\) and \(T\)) and feasible. Thus, let \(p^* \colon [T] \times \{0, \dotsc, T-1\} \to [0,1]\) be a function that attains \(\min_{p} V_p[0,0]\) and define \(V^* \coloneqq V_{p^*}\). The next proposition shows that \(V^*\) can be computed recursively and show how to obtain \(p^*\) from \(V^*\).



\begin{theorem}
  \label{thm:opt_dp_formula}
  For each \(t,g \in \iinterval{0}{T}\)
  \begin{align*}
    V^*[t,g] &= 0 & \qquad \text{if}~t = T~\text{or}~\text{g = T},
  \\
    V^*[t,g] &= \frac{1}{2} \paren*{V^*[t+1, g+1] + V^*[t+1, g-1]} & \text{if}~t < T~\text{and}~0< g < T,
  \\
  V^*[t,0] &= V^*[t+1, 1] + \frac{1}{2}&\text{if}~t < T.
  \end{align*}
  Furthermore, if we define \(p^* \colon \{0, \dotsc, T\}^2 \to [0,1]\) by
  \begin{equation*}
    p^*(t,g) \coloneqq
    \begin{cases}
      \frac{1}{2}(V_{p^*}[t, g - 1] - V_{p^*}[t, g+1]),
      & \text{if}~g > 0,\\
      \frac{1}{2} &\text{if}~{g = 0},
    \end{cases}
    \qquad \forall t \in [T], \forall g \in \{0, \dotsc, T-1\},
  \end{equation*}
  then \(V_{p^*} = V^*\).
\end{theorem}
\begin{proof}
  Let us show that \(p^*\) as defined in the statement of the theorem attains \(\inf_p V_p[0,0]\), where the infimum ranges over all functions from \([T] \times \iinterval{0}{T-1}\) to \([0,1]\). Note that smaller values of any entry of \(V_{p^*}\) can only make \(V_{p^*}[0,0]\) smaller. Thus, showing that \(p^*(t+1,g)\) minimizes \(V_{p^*}[t,g]\) for all \(t,g \in \{0,\dotsc, T-1\}\) (given that \(V_{p^*}[t',g']\) is fixed for \(t' \geq t+1\) and \(g' \in \iinterval{0}{T}\)) suffices to that \(p^*\) minimizes\footnote{One may fear that choosing \(p^*(t+1,g)\) to minimize \(V_{p^*}[t,g]\) might increase other entries, making this argument invalid. However, note that \(V_{p^*}[t,g]\) depends only on \(p^*(t+1,g)\) and entries \(V_{p^*}[t', g']\) with \(t' > t\). Thus, it is not hard to see that by proceeding from higher to smaller values of \(t \in \{0, \dotsc, T\}\), we can in fact pick \(p^*(t+1,g)\) to minimize \(V_{p^*}[t,g]\).} \(V_{p^*}[0.0]\). Moreover, by Theorem~\ref{thm:dp_formula} and the definition of \(p^*\) we have that \(V_{p^*}\) obeys the formulas on the statement of this theorem. Thus, we only need to show that this choice of \(p^*\) indeed minimizes the entries of \(V_{p^*}\).

  Let us first show that
  \begin{equation}
    \label{eq:claim1_opt_dp}
    V_{p^*}[t,g-1] - V_{p^*}[t, g+1] \in [0,1]~\text{for all}~g \in \iinterval{1}{T-1}~\text{and}~t \in \iinterval{0}{T}.
  \end{equation}
  For \(t = T\), since \(V_{p^*}[T, \cdot] \equiv 0\), the above claim clearly holds. Let \(t \in \iinterval{0}{T-1}\) and \(g \in \iinterval{1}{T-1}\). If \(g = 1\) we have
  \begin{align*}
    V_{p^*}[t,g-1] - V_{p^*}[t, g+1]
    &= V_{p^*}[t + 1, g] + \frac{1}{2} - \frac{1}{2}
    \paren*{V_{p^*}[t + 1, g] + V_{p^*}[t + 1, g+2]}
    \\
    &= \frac{1}{2}
    \paren*{V_{p^*}[t + 1, g]
    - V_{p^*}[t + 1, g+2]} + \frac{1}{2}.
  \end{align*}
  The last term above, by the induction hypothesis, is in \([0,1]\).
  Similarly, if \(T -1 \geq g \geq 2\) we have
  \begin{align*}
    &V_{p^*}[t,g-1] - V_{p^*}[t, g+1]
    \\
    &= \frac{1}{2}\paren*{V_{p^*}[t + 1, g] + V_{p^*}[t + 1, g - 2]} - \frac{1}{2}
    \paren*{V_{p^*}[t + 1, g] + V_{p^*}[t + 1, g+2]}
    \\
    &= \frac{1}{2}\paren*{V_{p^*}[t + 1, g - 2] - V_{p^*}[t + 1, g]} + \frac{1}{2}
    \paren*{V_{p^*}[t + 1, g] - V_{p^*}[t + 1, g+2]},
  \end{align*}
  and the last term is in \([0,1]\) by the induction hypothesis. This completes the proof of~\eqref{eq:claim1_opt_dp}.

  Let \(t \in [T]\) and \(g \in \iinterval{0}{T-1}\). Let us now show that \(p^*(t,g)\) minimizes \(V_{p^*}[t,g]\)  given that entries of the form \(V_{p^*}[t', g']\) for \(t' \geq t+1\) and \(g' \in \iinterval{0}{T}\) are fixed. For \(g = 0\), Theorem~\ref{thm:dp_formula} shows that \(V_{p^*}[t,0] = V_{p^*}[t+1, 1] + \max\{ 1 - \alpha, \alpha \}\) for some \(\alpha \in [0,1]\), which is minimized when \(\alpha = 1/2 = p^*(t,0)\). For \(g > 0\), Theorem~\ref{thm:dp_formula} tells us that
  \begin{equation*}
    V_{p^*}[t,g] = \max\curly{V_{p^*}[t+1, g+1] + p^*(t+1,g), V_{p^*}[t+1, g - 1] - p^*(t+1,g)}.
  \end{equation*}
  Since \(p^*(t+1,g) \geq 0\) and \(V^*[t+1, g+1] \leq V^*[t+1, g-1]\) by~\eqref{eq:claim1_opt_dp}, \(p^*(t+1,g)\) certainly minimizes \(V^*[t,g]\) since it makes both terms in the maximum equal. Finally,~\eqref{eq:claim1_opt_dp} guarantees that\footnote{In fact, it guarantees that \(p^*(t,g) \leq 1/2\). Intuitively this makes sense since we want to give more probability to the current best/leading expert than to the lagging expert.} \(p^*(t+1,g) \in [0,1]\).
\end{proof}

\subsection{Discrete Backwards Heat Equation}
\label{app:dbhe}

Interestingly, the potential function \(V^*\) for Cover's optimal algorithm satisfies the discrete backwards heat equation when the gap is not . For simplicity, let us focus on the simpler case with \(t \in [T]\) and gap \(g \in \iinterval{1}{T-1}\). Then, taking \(V[t, T+1] = 0\) we have
\begin{align*}
  V_t^*[t,g]
  &= V^*[t,g] - V^*[t-1, g]
  \\
  &\stackrel{}{=} V^*[t,g] - \frac{1}{2}(V^*[t, g+1] + V[t, g-1])
  & \text{(by Theorem~\ref{thm:opt_dp_formula})}
  \\ 
  &= -\frac{1}{2}(-2V^*[t,g] + V^*[t, g+1] + V[t, g-1])
  \\ 
  &= -\frac{1}{2}V_{gg}^*[t,g].
\end{align*}
The same holds for the case where the gap is zero, but we need to extend \(V^*\) for when the gap is \(-1\). Namely, set \(V[t, -1] = V[t, 1] + 1\) for \(t < T\). This guarantees that \(p(t,0) = 1/2 = (1/2)(V[t, -1] - V[t, 1])\) and \(V^*[t,0] = 1/2(V[t+1,1] + V[t+1,-1])\), so the cases with zero gap agree with the formulas for non-zero gaps in Theorem~\ref{thm:opt_dp_formula}. Interestingly, one may verify that \(p^*(t,g)\) also satisfies the discrete BHE by setting \(p(t, -g) = 1 - p(t,g)\) for \(g \geq 0\).

\subsection{Connecting the Regret with Random Walks}

As argued before, to give an upper-bound on the regret of \(\Acal_{p^*}\), where \(p^*\) is as in Theorem~\ref{thm:opt_dp_formula}, we need only to bound the value of \(V_{p^*}[0,0] = V^*[0,0]\). Interestingly, the entries of \(V^*\) have a strong connection to random walks, and this helps us give an upper-bound on the value of \(V^*[0,0]\). In the next theorem and for the remainder of the text, a \textbf{random walk} (of length \(t \in \Naturals\) starting at \(g \in \Integers_{\geq 0}\)) is a sequence of random variables \((S_i)_{i = 0}^t\) where \(S_i \coloneqq g + X_1 + \dotsm + X_i\) for each \(i \in \iinterval{0}{t}\) and \(\{X_j\}_{j \in [t]}\) are i.i.d.\ random variables taking values in \(\{\pm 1\}\). If we do not specify a starting point of a random walk, take it to be \(0\). We say that \(S_t\) is \textbf{symmetric} if \(\Prob(X_1 = 1) = \Prob(X_1 = -1) = 1/2\). Moreover, a \textbf{reflected random walk} (of length~\(t\)) is the sequence of random variables \((\abs{S_i})_{i \in \iinterval{0}{t}}\) where \((S_i)_{i = 0}^{t}\) is a random walk. Finally, we say that a random walk \((S_{i})_{i = 0}^t\)  \textbf{passes through} \(g \in \Naturals\) if the event \(\{S_i = g\}\) happens for some \(i \in \iinterval{0}{t}\).

The following lemma gives numeric bounds on the expected number of passages through \(0\) of a symmetric random walk. Its proof boils down to careful applications of Stirling's formula and can be found in Appendix~\ref{app:passages_through_0}. 

\begin{restatable}{lemma}{ExpectedPassagesThroughZero}
  \label{lemma:expected_passages_through_0}
  Let the random variable \(Z_T(0)\) be the number of passages through \(0\) of a reflected symmetric random walk of length \(T\). Then,
  \begin{equation*}
    \sqrt{\frac{2T}{\pi}} + \frac{3}{5} \leq  \Expect[Z_T(0)] \leq 1 + \sqrt{\frac{2T}{\pi}}.
  \end{equation*}
\end{restatable}

We are now in position to prove an upper-bound on the performance of \(\Acal_{p^*}\).

\begin{theorem}
  For every \(r,g \in \Naturals\), let the random variable \(Z_{r}(g)\) be the number of passages through 0 of a reflected symmetric random walk of length \(r\) starting at position \(g\). Then \(V^*[t,g] = \tfrac{1}{2}\Expect[Z_{T - t - 1}(g)]\) for every \(t,g \in \Naturals\). In particular,
  \begin{equation*}
    V^*[0,0] = \frac{1}{2}\Expect[Z_{T -1}(0)] \leq \sqrt{\frac{T}{2 \pi}}
  \end{equation*}
\end{theorem}
\begin{proof}
  Let us show that \(V^*[t,g] = (1/2)\Expect[Z_{T - t - 1}(g)]\) for all \(t,g \in \{0, \dotsc, T\}\) by induction on \(T - t\). For \(t = T\) we have \(Z_{T-t-1}(g) = 0\). Assume \(t = T-1\) and let \(g \in \{0, \dotsc, T\}\). If \(g > 0\) we have \((1/2)Z_0(g) = 0 = V^*[T-1, g]\). If \(g = 0\), we have \((1/2)Z_0 = 1/2 = V^*[T-1,0]\). Suppose now \(t < T-1\). If \(g = T\), then we have \(V^*[t,g] = 0 = (1/2)Z_{T - t - 1}(g)\). Now let us look at the case  \(T > g > 0\). By Theorem~\ref{thm:opt_dp_formula} and by the induction hypothesis, we have
  \begin{align*}
    &V^*[t,g]
    \\
    &= \frac{1}{2}\paren*{ V^*[t+1,g +1] + V^*[t+1, g-1]}
    \\
    &= \frac{1}{2}\paren*{  \frac{1}{2} \Expect[Z_{T - t -2}(g+1)] + \frac{1}{2}\Expect[Z_{T - t -2}(g - 1)]}
    \\
    &=
    \frac{1}{2}\Big(\Prob(S_{T - t - 2} = S_{T - t -1} + 1) \Expect[Z_{T - t -2}(g+1)] 
    \\
    &\qquad+ \Prob(S_{T - t - 2} = S_{T - t - 1} - 1)\Expect[Z_{T - t -2}(g - 1)]\Big)
    \\
    &= \frac{1}{2} \Expect[Z_{T - t}(g)].
  \end{align*}
  Similarly, for the case when \(g = 0\) we have
  \begin{equation*}
    V^*[t,0] = V^*[t+1,1] + \frac{1}{2}
    = \frac{1}{2} \Expect[Z_{T - t - 2}(1) + 1]
    =  \frac{1}{2} \Expect[Z_{T - t - 1}(0)].
  \end{equation*}
  In particular, we have \(V^*[0,0] = (1/2)\Expect[Z_T(0)]\) and Lemma~\ref{lemma:expected_passages_through_0} gives us the desired numerical bound.
\end{proof}

\subsection{Lower Bound on the Optimal Regret}
\label{sec:lower_bound}
 In the previous section we showed that Cover's algorithm suffers regret at most \(\sqrt{T/(2 \pi)} + O(1)\). In fact, by the definition of \(V_p\) (see~\eqref{eq:Vp_definition}) and \(V^*\) we have that \(V^*[0,0]\) is the minimum regret algorithms of the form \(\Acal_{p}\), where \(p \colon [T] \times \iinterval{0}{T-1} \to [0,1]\) is some function, suffer in the worst-case scenario. However, this does not tell us whether more general player strategies can do better or not. The next theorem shows that \emph{any} player strategy suffers, in the worst case, at least \(\sqrt{T/(2\pi)} - O(1)\) regret. The proof of the theorem boils down to lower-bounding the expected regret of a random adversary that plays uniformly from \(\Lcal = \curly{(0,1)^{\transp}, (1,0)^{\transp}}\).

\begin{restatable}{theorem}{RegretLowerBound}
  Let \(\Acal\) be a player strategy for a 2-experts game with \(T \in \Naturals\) rounds. Then, there is \(\ell \in \Lcal^{T}\) such that
  \begin{equation*}
    \Regret(T, \Acal, \Bcal_{\ell})
    \geq \sqrt{\frac{T}{2 \pi}} - \sqrt{\frac{1}{2\pi}} -\frac{1}{5}
  \end{equation*}
\end{restatable}
\begin{proof}
  Let \(\curly{\elltilde_t}_{t = 1}^T\) be i.i.d.\ random variables such that \(\elltilde_t\) is equal to a vector in \(\Lcal = \curly{(0,1)^{\transp}, (1,0)^{\transp}}\) chosen uniformly at random and let \(\Bcaltilde\) be the (randomized) oblivious adversary that plays \(\elltilde_t\) at round \(t\).   We shall show that
  \begin{equation}
    \label{eq:lower_bound_1}
    \Expect[\Regret(T,\Acal, \Bcaltilde)] \geq
    \sqrt{\frac{T}{2\pi}} - \sqrt{\frac{1}{2\pi}} -\frac{1}{5},
  \end{equation}
  which implies the existence of a deterministic adversary as described in the statement. For each \(t \in \iinterval{0}{T}\), let the random variable \(\gtilde_t\) be the gap between experts due to the costs of \(\Bcaltilde\) on round \(t\), define \(x_t \coloneqq \Acal(\elltilde_1, \dotsc, \elltilde_{t-1})\), and set \(p_t \coloneqq x_t(i_t)\) where \(i_t \in [2]\) is a lagging expert on round \(t\). It is worth to already note that \((\gtilde_t)_{t = 0}^T\) is a reflected random walk of length \(T\). By Proposition~\ref{prop:gap_regret} we have
  \begin{equation*}
    \Expect[\Regret(T)] = \sum_{t = 1}^T
    \Expect\sqbrac[\big]{ \boole{\gtilde_{t-1} > 0} p_t \cdot (\gtilde_t - \gtilde_{t-1})} + \sum_{t = 1}^T
    \Expect\sqbrac[\big]{\boole{\gtilde_{t-1} = 0} \iprodt{\elltilde_t}{x_t}},
  \end{equation*}
  where we recall that for any predicate \(P\) we have \(\boole{P}\) equals \(1\) if \(P\) is true, and equals \(0\) otherwise.
  First, let us show that
  \begin{equation}
    \label{eq:lower_bound_claim_1}
    \Expect\sqbrac[\big]{\boole{\gtilde_{t-1} > 0} p_t \cdot (\gtilde_t - \gtilde_{t-1})} = 0,
    \qquad \forall t \in [T].
  \end{equation}
  For each \(t \in \iinterval{0}{T}\), define \(\Expect_t[\cdot] \coloneqq  \Expect[ \cdot \rvert  \elltilde_1, \dotsc, \elltilde_t]  \), that is, \(\Expect_t\) is the conditional expectation given the choices of the random adversary on rounds \(1, \dotsc, t\). Let \(t \in [T]\). On the event \(\curly{ \gtilde_{t - 1} > 0}\), one  can see that \(\gtilde_t - \gtilde_{t-1}\) is independent of \(\elltilde_1, \dotsc, \elltilde_{t-1}\) and is uniformly distributed on \(\{\pm 1\}\). This together with the fact that \(p_t\) is a function of \(\elltilde_1, \dotsc, \elltilde_{t-1}\) implies
  \begin{align*}
    \Expect\sqbrac[\big]{ \boole{\gtilde_{t-1} > 0} p_t \cdot (\gtilde_t - \gtilde_{t-1}) }
    &= \Expect\sqbrac[\Big]{
    \Expect_{t-1}\sqbrac[\big]{ \boole{\gtilde_{t-1} > 0} p_t \cdot (\gtilde_t - \gtilde_{t-1}) }
    }
    \\
    &=
    \Expect\sqbrac[\Big]{ p_t
    \Expect_{t-1}\sqbrac[\big]{ \boole{\gtilde_{t-1} > 0}  (\gtilde_t - \gtilde_{t-1}) }
    }
    \\
    &=
    \Expect\sqbrac[\Big]{ p_t
    \Expect\sqbrac[\big]{ \boole{\gtilde_{t-1} > 0}  (\gtilde_t - \gtilde_{t-1}) }
    }
    \\
    &=
    \Expect\sqbrac{ p_t
    \cdot 0 } = 0.
  \end{align*}
  This ends the proof of~\eqref{eq:lower_bound_claim_1}.
  Let us now show that
  \begin{equation}
    \label{eq:lower_bound_claim_2}
    \sum_{t = 1}^T
    \Expect\sqbrac[\big]{\boole{\gtilde_{t-1} = 0} \iprodt{\elltilde_t}{x_t}}
    = \frac{1}{2}\Expect[Z_{T-1}(0)].
  \end{equation}
  For each \(t \in [T]\), since \(x_t\) is a function of \(\elltilde_1, \dotsc, \elltilde_{t-1}\) and \(\elltilde_t\) is independent of \(\elltilde_1, \dotsc, \elltilde_{t-1}\), we have
  \begin{align*}
    \Expect\sqbrac[\big]{\boole{\gtilde_{t-1} = 0} \iprodt{\elltilde_t}{x_t}}
    &= \Expect\sqbrac[\Big]{\Expect_{t-1}\sqbrac[\big]{\boole{\gtilde_{t-1} = 0} \iprodt{\elltilde_t}{x_t}}}
    \\
    &= \Expect\sqbrac[\Big]{\boole{\gtilde_{t-1} = 0}\iprodt{ \Expect_{t-1}\sqbrac[\big]{ \elltilde_t}}{x_t}}
    \\
    &= \Expect\sqbrac[\Big]{\boole{\gtilde_{t-1} = 0}\iprodt{ \Expect\sqbrac[\big]{ \elltilde_t}}{x_t}}
    \\
    &= \Expect\sqbrac[\Big]{\boole{g_{t-1} = 0}\paren[\Big]{\frac{1}{2}x_t(1) + \frac{1}{2}x_t(2)}}
    = \frac{1}{2} \Prob(\gtilde_{t-1} = 0).
  \end{align*}
  Thus,
  \begin{equation*}
    \sum_{t = 1}^T
    \Expect\sqbrac[\big]{\boole{\gtilde_{t-1} = 0} \iprodt{\elltilde_t}{x_t}} = \frac{1}{2} \sum_{t = 1}^{T} \Prob(\gtilde_{t-1} = 0)
    = \frac{1}{2} \Expect\sqbrac[\Big]{\sum_{t = 1}^{T} \indic{\curly{\gtilde_{t-1} = 0}}}
    = \frac{1}{2} \Expect[Z_{T-1}(0)].
  \end{equation*}
  This completes the proof of~\eqref{eq:lower_bound_claim_2} and the desired numerical lower-bound is given by Lemma~\ref{lemma:expected_passages_through_0}.
\end{proof}


\section{On the Passages Through Zero of a Symmetric Random Walk}
\label{app:passages_through_0}

In this section we shall prove Lemma~\ref{lemma:expected_passages_through_0}, which bounds the expected number of passages through \(0\) of a symmetric random walk. First, we need a simple corollary of Stirling's formula (which we state here for convenience) to bound binomial terms.

\begin{theorem}[{Stirling's Formula, \citealp{Robins55a}}]
  For any \(n \in \Naturals\) we have
  \begin{equation*}
    \sqrt{2 \pi n } \paren*{\frac{n}{e}}^n e^{1/(12n + 1)} < n! < \sqrt{2 \pi n } \paren*{\frac{n}{e}}^n e^{1/(12n)}.
  \end{equation*}
\end{theorem}

\begin{corollary}
  \label{cor:binom_bound}
  For any \(n \in \Naturals\) we have
  \begin{equation*}
    \frac{2^{2n}}{\sqrt{\pi n}}\paren*{1 - \frac{2}{15 n}}  \leq \binom{2n}{n} \leq \frac{2^{2n}}{\sqrt{\pi n}}.
  \end{equation*}
\end{corollary}
\begin{proof}
  Let \(n \in \Naturals\). For the upper-bound, we have
  \begin{align*}
    \binom{2n}{n}
    &=
    \frac{2n!}{(n!)^2}
    \\
    &< \frac{\sqrt{2\pi n}\cdot (2n)^{2n} \cdot e^{-2n} \cdot e^{1/24n}}{
      2\pi n \cdot n^{2n} \cdot e^{-2n}\cdot e^{2/(12n + 1)}
    }
    \\
    &= \frac{2^{2n} \sqrt{2}}{\sqrt{2 \pi n}} \cdot
    \exp\paren*{\frac{1}{24n} - \frac{2}{12n + 1}}
    \\
    &\leq \frac{2^{2n} }{\sqrt{\pi n}} \cdot
    \exp\paren*{-\frac{1}{24n}} \leq \frac{2^{2n} }{\sqrt{\pi n}}.
  \end{align*}
  Similarly, for the lower-bound we have
  \begin{align*}
    \binom{2n}{n}
    &=
    \frac{2n!}{(n!)^2}
    \\
    &> \frac{\sqrt{2\pi n}\cdot (2n)^{2n} \cdot e^{-2n} \cdot e^{1/(24n + 1)}}{
      2\pi n \cdot n^{2n} \cdot e^{-2n}\cdot e^{2/12n}
    }
    \\
    &=
    \frac{2^{2n} \sqrt{2}}{\sqrt{2 \pi n}} \cdot
    \exp\paren*{\frac{1}{24n} - \frac{1}{6n}}
    \\
    &= \frac{2^{2n}}{\sqrt{\pi n}} \cdot
    \exp\paren*{-\frac{4}{30n}}
    \\
    &\geq  \frac{2^{2n}}{\sqrt{\pi n}}
    \paren*{1 -\frac{2}{15n}},
    &\text{(Since~\(e^{-x} \geq 1 - x\) for \(x \geq 0\)).} &&
  \end{align*}
\end{proof}

We are now ready to prove Lemma~\ref{lemma:expected_passages_through_0}, which we restate for convenience.

\ExpectedPassagesThroughZero*

\begin{proof}
  Let \(\{S_t\}_{t = 0}^T\) be a symmetric random walk and define \(X_i \coloneqq S_i - S_{i -1}\) for every \(i \in [T]\). Note that
  \begin{align*}
    \Prob(\abs{S_t} = 0)
    = \Prob(S_t = 0)
    &= \Prob\paren*{
      \card{\setst{X_i = 1}{i \in [t]}}
      = \card{ \setst{X_i = -1}{i \in [t]}}
    }
    \\
    &=
    \begin{cases}
      0 &\text{if}~t~\text{is odd},\\
      \binom{t}{t/2} 2^{-t}
      &\text{if}~t~\text{is even}.
    \end{cases}
  \end{align*}
  Therefore,
  \begin{equation*}
    \Expect[Z_{T}(0)]
    = \sum_{t = 0}^{T} \Prob(S_t = 0)
    = \sum_{k = 0}^{\floor{T/2}} \Prob(S_{2k} = 0)
    = \sum_{k = 0}^{\floor{T/2}} \binom{2k}{k}\frac{1}{2^{2k}}.
  \end{equation*}
  Using Corollary~\ref{cor:binom_bound}, a consequence of Stirling's approximation to the factorial function, we can show upper- and lower-bounds to the above quantity. Namely, for the upper-bound we have
  \begin{align*}
    \sum_{k = 0}^{\floor{T/2}} \binom{2k}{k}\frac{1}{2^{2k}}
     &\leq 1 + \sum_{k = 1}^{\floor{T/2}}
     \frac{2^{2k}}{\sqrt{\pi k}} \frac{1}{2^{2k}}
     = 1 + \frac{1}{\sqrt{\pi}}
     \paren[\Bigg]{ \sum_{k = 1}^{\floor{T/2}}\frac{1}{\sqrt{k}}
     } \leq
     1 + \frac{1}{\sqrt{\pi}}
     \paren[\Bigg]{ \int_{0}^{T/2}\frac{1}{\sqrt{x} \diff x}
     }
     \\
     &= 1 + \paren*{2\sqrt{\frac{k}{\pi}}} \Bigg\rvert_{k = 0}^{T/2}
     = 1 + \sqrt{\frac{2T}{\pi}}.
   \end{align*}
  We proceed similarly for the lower-bound. By setting \(\beta \coloneqq \sum_{k = 1}^{\infty} k^{-3/2}\) we get
  \begin{align*}
   \sum_{k = 0}^{\floor{T/2}} \binom{2k}{k}\frac{1}{2^{2k}}
    &\geq 1 + \sum_{k = 1}^{\floor{T/2}}
    \frac{2^{2k}}{\sqrt{\pi k}}\paren*{1 - \frac{2}{15k}} \frac{1}{2^{2k}}
    = 1 + \frac{1}{\sqrt{\pi}}
    \paren[\Bigg]{ \sum_{k = 1}^{\floor{T/2}}\frac{1}{\sqrt{k}} - \sum_{k = 1}^{\floor{T/2}}\frac{2}{15k^{3/2}}
    }
    \\
    &\leq
    1 + \frac{1}{\sqrt{\pi}}
    \paren[\Bigg]{ \int_{0}^{T/2}\frac{1}{\sqrt{x} \diff x} -  \frac{2}{15}\sum_{k = 1}^{\infty} k^{-{3/2}}
    }
    = \paren*{2\sqrt{\frac{k}{\pi}}} \Bigg\rvert_{k = 0}^{T/2} - \frac{2}{15}\beta + 1
    \\
    &= \sqrt{\frac{2T}{\pi}} - \frac{2}{15}\beta + 1.
  \end{align*}
  To conclude the proof, some simple calculations yield
  \begin{equation*}
    \beta = \sum_{k = 1}^{\infty} \frac{1}{k^{3/2}}
    = 1 + \sum_{k = 2}^{\infty} \frac{1}{k^{3/2}}
    \leq
    1 + \int_{1}^{\infty} \frac{1}{x^{3/2}} \diff x
    = 1 + \paren*{-\frac{2}{\sqrt{x}}}\Bigg\rvert_{x = 1}^\infty = 3.
  \end{equation*}
\end{proof}

\section{Missing Proofs for Section~\ref{sec:continuous_problem}}
\label{app:missing_proofs_continuous_regret}

\BHEForQ*
\begin{proof}
  Fix \(t \in [0,T)\) and \(g \in \Reals_{\geq 0}\). Then,
  \begin{align*}
    \partial_t Q(t,g)
    &= \frac{1}{2} \partial_t \erfc\paren*{\frac{g}{\sqrt{2(T - t)}}}
    = - \frac{1}{\sqrt{\pi}} \exp\paren*{- \frac{g^2}{2(T - t)}}
    \partial_t \paren*{\frac{g}{\sqrt{2(T - t)}}}
    \\
    &=
    - \frac{1}{\sqrt{\pi}} \exp\paren*{- \frac{g^2}{2(T - t)}}
    \frac{g}{(2(T - t))^{3/2}}.
  \end{align*}
  Similarly,
  \begin{align*}
    \partial_g Q(t,g)
    &= \frac{1}{2} \partial_g \erfc\paren*{\frac{g}{\sqrt{2(T - t)}}}
    = - \frac{1}{\sqrt{\pi}} \exp\paren*{- \frac{g^2}{2(T - t)}}
    \partial_g \paren*{\frac{g}{\sqrt{2(T - t)}}}
    \\
    &=
    - \frac{1}{\sqrt{\pi}} \exp\paren*{- \frac{g^2}{2(T - t)}}
    \frac{1}{\sqrt{2(T - t)}}
  \end{align*}
  and
  \begin{align*}
    \partial_{gg} Q(t,g)
    &= - \frac{1}{\sqrt{2\pi(T - t)}} \partial_g
     \exp\paren*{- \frac{g^2}{2(T - t)}}
     \\
     &
     = \frac{1}{\sqrt{2\pi(T - t)}}
     \exp\paren*{- \frac{g^2}{2(T - t)}}  \frac{g}{T - t}
     = - 2 \partial_t Q(t,g),
    \end{align*}
    as desired.
\end{proof}

Let us now prove~\eqref{eq:rtg_formula}.
\begin{lemma}
  For all \(t \in (-\infty, T)\) and \(g \in \Reals\), we have
  \begin{equation*}
    R(t,g) = \frac{g}{2} \erfc\paren[\Bigg]{\frac{g}{\sqrt{2(T - t)}}} - \sqrt{\frac{T -t}{2 \pi}} \exp\paren[\Bigg]{-\frac{g^2}{2(T - t)}} + \sqrt{\frac{T}{2 \pi}}.
  \end{equation*}
\end{lemma}
\begin{proof}
  Fix \((t,g) \in (-\infty, T) \times \Reals\). Using that \(\int_0^g \erfc(x) \diff x = g \erfc(g) - \frac{1}{\sqrt{\pi}} e^{-g^2} + \frac{1}{\sqrt{\pi}}\) \citep[Section~7.7(i)]{NIST10a}, we have
 \begin{align*}
  \int_0^g Q(t,x) \diff x 
  &= 
  \frac{\sqrt{2(T-t)}}{2}\int_{0}^{\nicefrac{g}{\sqrt{2(T -t)}}}
  \erfc(y) \diff y
  \\
  &= \frac{\sqrt{2(T-t)}}{2}\paren*{x \erfc(x) - \frac{e^{-x^2}}{\sqrt{\pi}}}\Bigg|_{x = 0}^{x = \nicefrac{g}{\sqrt{2(T - t)}}}
  \\
  &= \frac{g}{2} \erfc\paren[\Bigg]{\frac{g}{\sqrt{2(T - t)}}} - \sqrt{\frac{T -t}{2 \pi}} \exp\paren[\Bigg]{-\frac{g^2}{2(T - t)}}
  + \sqrt{\frac{T - t}{2\pi}}
 \end{align*}
 for \(g \geq 0\). Similarly, using that \(\frac{\diff}{\diff x} \erfc(x) = - \frac{2}{\sqrt{\pi}} e^{-x^2} \) we have
 \begin{equation*}
   \frac{1}{2}\int_0^t \partial_g Q(s,0) \diff s = \frac{1}{2}\int_0^t \paren*{- \frac{1}{\sqrt{2\pi(T - s)}}} \diff s = \sqrt{\frac{T - t}{2\pi}} - \sqrt{\frac{T}{2\pi}}.
 \end{equation*}
 Plugging both equations into the definition of \(R\) concludes the proof.
\end{proof}

\section{Missing Proofs for Section~\ref{sec:discretization}}
\label{sec:missing_proofs_discretization}

Let us begin by proving a crucial condition on the function \(q\) defined on Section~\ref{sec:discretization}.

\qIsOneHalf*
\begin{proof}
  The claim follows directly from the definition of \(q\) for \(t = T\). Let \(t \in \iinterval{0}{T-1}\). From the definition of \(q\), we have
  \begin{align*}
    q(t,0)
    &= R_g(t,0)
    = \frac{1}{2}(R(t, 1) - R(t, -1))\\
    &= \frac{1}{2}\paren*{\int_0^1 Q(t,x) \diff x - \frac{1}{2} \int_0^t \partial_g Q(s,0) \diff s
    - \int_0^{-1} Q(t,x) \diff x + \frac{1}{2} \int_0^t \partial_g Q(s,0) \diff s}
    \\
    &= \frac{1}{2}\paren*{\int_0^1 Q(t,x) \diff x
    - \int_0^{-1} Q(t,x) \diff x}
    \\
    &= \frac{1}{2}\paren*{\int_0^1 Q(t,x) \diff x
    + \int_0^{1} Q(t,-x) \diff x}
    \\
    &= \frac{1}{2}\int_0^1 \paren[\big]{Q(t,x) + Q(t,-x)} \diff x.
  \end{align*}
  Moreover, note that for any \(z \in \Reals\) we have
  \begin{equation}
    \label{eq:negative_erfc_property}
    \erfc(-z)
    = 1 - \frac{2}{\sqrt{\pi}}\int_{0}^{-z} e^{-x^2}\diff x
    = 1 + \frac{2}{\sqrt{\pi}}\int_{0}^{z} e^{-x^2}\diff x
    = 2 - \erfc(z).
  \end{equation}
  Therefore, \(Q(t,-x) = 1 - Q(t,x)\) for any \(x \in \Reals\) and
  \begin{equation*}
    q(t,0)
    = \frac{1}{2}\int_0^1 \paren[\big]{Q(t,x) + Q(t,-x)} \diff x
    = \frac{1}{2}\int_0^1 1 \diff x
    = \frac{1}{2}.
  \end{equation*}
\end{proof}

This section contains the proofs on the bounds on \(r_t\) and \(r_{gg}\). We start by bounding \(r_t\).

\BoundOnrt*
\begin{proof}
  Fix \(t \in (-\infty, T)\) and \(g \in \Reals\). Note that \(R\) is continuously differentiable
  with respect to its first argument on \([t-1, t]\)
  and two times differentiable on \((t-1, t)\). Thus, by Taylor's Theorem, there
  is \(t' \in (t-1, t)\) such that
  \begin{equation*}
    R(t-1, g) = R(t,g) + (-1)\partial_t R(t,g)
    + (-1)^2 \frac{\partial_{tt}  R(t',g) }{2},
  \end{equation*}
  where \(\partial_{tt} R(t', g)\) denotes the second derivative of \(R\) with respect to its first argument at \((t', g)\).  Therefore,
  \begin{equation*}
    r_t(t,g) = \partial_t R(t,g) - (R(t,g) - R(t-1, g))
    = \frac{\partial_{tt} R(t',g)}{2}.
  \end{equation*}
  Thus, to bound \(r_t(t,g)\) we need only to bound \((1/2)\partial_{tt} R(t',g)\). Computing the derivatives yields
  \begin{equation*}
    \partial_t R(t',g)
    = \frac{1}{2 \sqrt{2 \pi (T - t')}} \exp \paren[\Big]{
      -\frac{g^2}{2(T - t')}}
  \end{equation*}
  and
  \begin{align*}
    \partial_{tt} R(t',g)
    &= \frac{\sqrt{2}}{8 \sqrt{\pi} (T - t')^{5/2}}
      \exp\paren[\Big]{-\frac{g^2}{2(T - t')}}
      \paren*{T - t' - g^2}
    \\
    &\leq \frac{\sqrt{2}}{8 \sqrt{\pi} (T - t')^{5/2}}
    \exp\paren[\Big]{-\frac{g^2}{2(T - t')}}
    \paren*{T - t'}
    \\
    &\leq \frac{\sqrt{2}}{8 \sqrt{\pi} (T - t')^{3/2}}
    &\text{(Since~\(T - t' > 0\)),}
    \\
    &\leq \frac{\sqrt{2}}{8 \sqrt{\pi} (T - t)^{3/2}}
    &\text{(Since~\(T - t' > T - t\)).}\qquad
  \end{align*}

To bound \(r_{gg}\), we will need to be slightly more careful. First, we will need the following simple lemma about Lipschitz continuity of \(x \in \Reals \mapsto xe^{x^2}\).

\begin{lemma}
  \label{lemma:lip_cont_exp}
  Let \(K > 0\) and define \(f(\alpha) \coloneqq \alpha e^{-\alpha^2/K}\) for every \(\alpha \in \Reals\). Then \(f\) is \(2\)-Lipschitz continuous.
\end{lemma}
\begin{proof}
  Let \(\alpha \in \Reals\). First, note that \(f'(\alpha) =
  e^{-\alpha^2/K}(1 - 2\alpha^2/K)\). Therefore, using the fact
  that \(e^\beta \geq 1 + \beta\) for any \(\beta \in \Reals\) we have\
  \begin{equation*}
    \abs{f'(\alpha)}
    = \frac{1}{\exp\paren*{\frac{\alpha^2}{K}}}\abs[\Bigg]{1 - \frac{2\alpha^2}{K}}
    \leq
    \frac{1}{\abs[\Big]{1 +\frac{\alpha^2}{K}}}\abs[\Bigg]{1 - \frac{2\alpha^2}{K}}
    \leq 2 \frac{\abs[\Big]{1 - \frac{\alpha^2}{K}}}{\abs[\Big]{1 +\frac{\alpha^2}{K}}} \leq 2.
  \end{equation*}
\end{proof}

We are now ready to bound \(r_{gg}\).
  Fix \(t \in (-\infty,T)\) and \(g \in \Reals\). Moreover, denote by $\partial_g^{(3)} R$ the third partial derivative of \(R\) with respect to its third argument. By Taylor's Theorem, there are
  \(g_+' \in (g, g+1)\) and \(g_{-}' \in (g-1, g)\) such that
  \begin{align*}
    R(t, g + 1)
    &= R(t,g) + \partial_g R(t,g) + \frac{1}{2}\partial_{gg} R(t,g)
    + \frac{1}{3!}\partial_g^{(3)} R(t, g_+')\quad\text{and}
    \\
    R(t, g - 1)
    &= R(t,g) - \partial_g R(t,g) + \frac{1}{2}\partial_{gg} R(t,g)
    - \frac{1}{3!}\partial_g^{(3)} R(t, g_{-}').
  \end{align*}
  Therefore,
  \begin{equation*}
    r_{gg}(t,g)
    = \partial_{gg}R(t,g) - (R(t, g+1) + R(t, g-1) - 2R(t,g))
    = \frac{1}{3!}(\partial_g^{(3)} R(t, g_{-}') - \partial_g^{(3)} R(t, g_+')).
  \end{equation*}
  Let \(g' \in \Reals\). To compute the partial derivatives, first
  note that
  \begin{equation*}
    \partial_g R(t,g') = Q(t,g') = \frac{1}{2}
  \erfc(g'/\sqrt{2(T - t)}).
  \end{equation*}
  Thus, one may check that
  \begin{equation}
    \label{eq:R_gap_derivatives}
    \begin{aligned}
      \partial_{gg} R(t,g')
      &= -\frac{1}{2}\frac{2}{\sqrt{\pi}} \exp\paren[\Big]{-\frac{(g')^2}{2(T-t)}} \frac{1}{\sqrt{2 (T - t)}}
      \\
      &= -\frac{1}{\sqrt{2\pi(T - t)}} \exp\paren[\Big]{-\frac{(g')^2}{2(T-t)}}
      \\ &\quad \text{and}
      \\
      \partial_{g}^{(3)} R(t,g')
      &
      = \frac{1}{\sqrt{2\pi (T - t)}} \exp\paren[\Big]{-\frac{(g')^2}{2(T-t)}} \frac{2g'}{2 (T - t)}
      \\
      &= \frac{1}{\sqrt{2 \pi}} \frac{g'}{(T -t)^{3/2}} \exp\paren[\Big]{-\frac{(g')^2}{2(T-t)}}.
    \end{aligned}
  \end{equation}
  By Lemma~\ref{lemma:lip_cont_exp}, we know that $\partial_{g}^{(3)}
  R(t,\cdot)$ is Lipschtiz continuous with Lipschitz constant \(2(2 \pi)^{-1/2} (T-t)^{-3/2}\). Therefore,
  \begin{equation*}
    r_{gg}(t,g)
    = \frac{1}{3!}(\partial_g^{(3)} R(t, g_{-}') - \partial_g^{(3)} R(t, g_+'))
    \leq \frac{2}{3\sqrt{2 \pi}(T - t)^{3/2}}\abs{g_{-}' - g_{+}'}
    \leq \frac{2\sqrt{2}}{3 \sqrt{\pi} (T - t)^{3/2}}.
  \end{equation*}

\end{proof}

\section{Extending the Regret Analysis for General Costs}
\label{app:general_costs}

In Section~\ref{sec:discretization}, we relied on the fact the gap values were in \(\{+1,-1\}\). This assumption was fundamental for the version of the discrete It\^o's Formula that
we have uses (see the assumption on \(g_0, \dotsc, g_T\) in the statement of Theorem~\ref{thm:discrete_ito}). It was also required by Proposition~\ref{prop:gap_regret} to connect the regret with the ``discrete stochastic integral''. To extend the upper-bound on the regret of the algorithm from Section~\ref{sec:discretization} to general costs we will follow the same techniques used by Harvey et al.\ to extend the guarantees of their algorithm to general costs: we shall use a more general version of the discrete It\^o's formula, concavity of \(R\) with respect to its second argument, and a lemma relating the per-round regret to terms that appear in the more general version of the discrete It\^o's formula (see~\citealp[Section~3.3]{HarveyLPR20a} for details on these arguments).

As in the work of Harvey et al., we will rely on a more general version of the discrete It\^o's formula that holds for general \([0,1]\) costs. The main issue with this general formula is that more work is needed to relate it to the regret of our player strategy.
\begin{theorem}[{General Discrete It\^o's Formula, \citealp[Lemma~3.13]{HarveyLPR20a_arxiv}}]
  \label{thm:general_discrete_ito}
  Let \(f \colon \Reals^2
  \to \Reals\) be a function and let \(g_0, g_1, \dotsc, g_T \in \Reals\). Then,
  \begin{align*}
    f(T, g_T) - f(0, g_0)
    = \sum_{t = 1}^T &\paren[\Big]{f(t, g_{t}) - \frac{f(t,g_{t-1} + 1) + f(t, g_{t-1} - 1)}{2}}
    \\
    &+ \sum_{t = 1}^T \paren[\big]{\tfrac{1}{2} f_{gg}(t, g_{t-1}) +
    f_{t}(t, g_{t-1})}.
  \end{align*}
\end{theorem}

Fix \(T \in \Naturals\), fix gaps \(g \in \Reals^T\), and set \(g_0 \coloneqq
0\). For the remainder of this section all results will be regarding a game of \(\Acal_q\) against an oblivious  adversary with gap sequence \(g_0, g_1, \dotsc, g_T \in \Reals_{\geq 0}\).

For every \(t \in \iinterval{0}{T}\), define the \textbf{per-round regret} (at
round \(t\)) by
\begin{equation*}
  \Delta_{\Regret}(t) \coloneqq \Regret(t)
  - \Regret(t-1).
\end{equation*}
Our goal in this section is to prove the following lemma.
\begin{lemma}
  \label{lemma:delta_regret_bound}
  For every \(t \in [T]\) we have
  \begin{equation}
    \label{eq:delta_regret_bound}
    \Delta_{\Regret}(t) \leq R(t, g_{t}) - \frac{R(t,g_{t-1} + 1) + R(t, g_{t-1} - 1)}{2}, \qquad \forall t \in [T].
  \end{equation}
\end{lemma}
Combining the above lemma with Theorem~\ref{thm:general_discrete_ito} and the fact that \(R\) satisfies~\eqref{eq:bhe} yields
\begin{align*}
  \Regret(T)
  &= \sum_{t = 1}^{T} \Delta_{\Regret}(t)
  \leq R(T-1, g_{T-1}) + \Delta_{\Regret}(T) + \sum_{t = 1}^{T-1}(\tfrac{1}{2}r_{gg}(t,g_{t-1}) + r_{t}(t,g_{t-1})).
\end{align*}
Since \(q(T, g) = \boole{g = 0} (1/2)\) for any \(g \in \iinterval{0}{T-1}\), we have \(\Delta_{\Regret}(T) \leq 1/2\). At this point, the exact same proof of Theorem~\ref{thm:final_theorem} applies and we obtain the same regret bound. Thus, it only remains to prove Lemma~\ref{lemma:delta_regret_bound}. In order to prove Lemma~\ref{lemma:delta_regret_bound}, we will use the following result from~\cite{HarveyLPR20a_arxiv}.

\begin{proposition}[{\citealp[Lemma~3.14]{HarveyLPR20a_arxiv}}]
  \label{prop:delta_regret}
  Let \(g_{t-1}\) and \(g_t\) be the values of the gap on rounds \(t-1\)
  and \(t\), respectively, and let  \(q(t, g_{t-1})\) be the probability
  mass put in the worst expert at round \(t\) by the player (with
  \(q(t,0) = 1/2\)). For all \(t \geq 1\),
  \begin{enumerate}
    \item \label{item:case_1} If a best expert at time \(t - 1\) remains a best expert at time \(t\), then,
    \begin{equation*}
      \Delta_{\Regret}(t) = q(t, g_{t-1})(g_t - g_{t-1}).
    \end{equation*}
    \item \label{item:case_2}  If a best expert at time \(t - 1\)
    remains a best expert at time \(t\), then \(g_t + g_{t-1} \leq 1\)
    and
    \begin{equation*}
      \Delta_{\Regret}(t) = g_t - q(t, g_{t-1})(g_t + g_{t-1}).
    \end{equation*}
  \end{enumerate}
\end{proposition}

We shall also make use of the following fact about concave function.
\begin{lemma}
  \label{fact:min_concave_function}
  Let \(f \colon \Reals \to \Reals\) be a concave function and let \(\alpha < \beta \) be real numbers. Then \(f(x) \geq \min \{f(\alpha), f(\beta)\}\).
\end{lemma}

\begin{proof}[{{Proof of Lemma~\ref{lemma:delta_regret_bound}}}]
  To
  prove~\eqref{eq:delta_regret_bound}, we will consider each one of the
  cases from Proposition~\ref{prop:delta_regret} separately.
  \paragraph{\textbf{Case~\ref{item:case_1}}.} In this
  case,~\eqref{eq:delta_regret_bound} is equivalent to
  \begin{equation}
    \label{eq:general_gaps_case_1}
    0 \leq - q(t, g_{t-1})(g_t - g_{t-1}) + R(t, g_t) - \frac{R(t, g_{t-1} + 1) + R(t, g_{t-1} - 1) }{2}.
  \end{equation}
  Since the first term in the right-hand side of the above inequality is
  linear in \(g_t\) and since \(R(t, \cdot )\) is concave
  (by~\eqref{eq:R_gap_derivatives} we know that \(\partial_{gg} R(t,
  \cdot)\) is negative everywhere), we conclude that the whole right-hand
  side is concave as a function of \(g_t\). Thus, by Fact~\ref{fact:min_concave_function} it suffices to prove
  the above inequality for \(g_t \in \{g_{t-1} -1, g_{t-1} + 1\}\) to prove that it
  holds for \(g_t \in [g_{t-1} -1, g_{t-1} + 1]\). But for \(g_t \in
  \{g_{t-1} -1, g_{t-1} + 1\}\) the right-hand side
  of~\eqref{eq:delta_regret_bound} becomes exactly \(q(t, g_{t-1})(g_t -
  g_{t-1})\).
  \paragraph{\textbf{Case~\ref{item:case_2}}.} In this
  case,~\eqref{eq:delta_regret_bound} is equivalent to
  \begin{equation}
    \label{eq:general_gaps_case_2}
    0 \leq - g_t + q(t, g_{t-1})(g_t + g_{t-1}) + R(t, g_t) - \frac{R(t, g_{t-1} + 1) + R(t, g_{t-1} - 1)}{2}.
  \end{equation}
  Again, the right-hand side of the above inequality is concave as a
  function of \(g_t\). Since \(g_t \geq 0\) and \(g_t + g_{t-1} \leq 1\),
  we know that \(g_t \in [0, 1 - g_{t-1}]\). Thus, it suffices to prove
  the above inequality for \(g_t \in \{0, 1 - g_{t-1}\}\). For \(g_t = 0\)
  we have
  \begin{align*}
    &- g_t + q(t, g_{t-1})(g_t + g_{t-1}) + R(t, g_t) - \frac{R(t, g_{t-1} + 1) + R(t, g_{t-1} - 1) }{2}\\
    = &- q(t, g_{t-1})(g_t - g_{t-1}) +  R(t, g_t) - \frac{R(t, g_{t-1} + 1) + R(t, g_{t-1} - 1) }{2},
  \end{align*}
  and in the previous case we showed that the above is non-negative for
  all \(g_t \in [g_{t-1} - 1, g_{t-1} + 1]\). Since \(g_{t-1} \leq 1\) in
  this case, we have in particular that the above holds for \(g_t = 0\).
  Suppose now that \(g_{t} = 1 - g_{t-1}\).
  Since \(q(t, g_{t-1}) = R_{g}(t, g_{t-1})\), we have that~\eqref{eq:general_gaps_case_2} is equivalent to
  \begin{align*}
      0 &\leq
      - g_t + q(t, g_{t-1})(g_t + g_{t-1}) + R(t, g_t) - \frac{R(t, g_{t-1} + 1) + R(t, g_{t-1} - 1)}{2}
      \\
      &= g_{t-1} - 1 + q(t, g_{t-1}) + R(t, 1 - g_{t-1}) - \frac{R(t, g_{t-1} + 1) + R(t, g_{t-1} - 1)}{2}
      \\
      &=
      \begin{aligned}[t]
        g_{t-1} - 1 &+ \frac{R(t, g_{t-1} + 1) - R(t, g_{t-1} - 1)}{2}+ R(t, 1 - g_{t-1}) 
        \\
          - &\frac{R(t, g_{t-1} + 1) + R(t, g_{t-1} - 1)}{2}
      \end{aligned}
      \\
      &= g_{t-1} - 1 + R(t, 1 - g_{t-1}) - R(t, g_{t-1} - 1).
  \end{align*}
  By the definition of \(R\) and since \(\erfc(-z) = 2 -\erfc(z)\) for all \(z \in \Reals\) (see~\eqref{eq:negative_erfc_property}), we have
  \begin{align*}
    R(t, 1 - g_{t-1}) - R(t, g_{t-1} - 1)
    &= \frac{1 - g_{t-1}}{2} \erfc\paren*{\frac{1 - g_{t-1}}{\sqrt{2(T-t)}}} - \frac{g_{t-1} - 1}{2} \erfc\paren*{\frac{g_{t-1} - 1}{\sqrt{2(T-t)}}}
    \\
    &= \paren*{\frac{1 - g_{t-1}}{2}}\paren*{\erfc\paren*{\frac{1 - g_{t-1}}{\sqrt{2(T-t)}}} + \erfc\paren*{\frac{g_{t-1} - 1}{\sqrt{2(T-t)}}}}
    \\
    &=
    \paren*{\frac{1 - g_{t-1}}{2}} 2
    = 1 - g_{t-1}.
  \end{align*}
   This concludes the proof of~\eqref{eq:general_gaps_case_2}.
\end{proof}

\end{document}